\documentclass{article}
\usepackage{arxiv}
\usepackage{multirow}
\usepackage[utf8]{inputenc} % allow utf-8 input
\usepackage[T1]{fontenc}    % use 8-bit T1 fonts
\usepackage{url}            % simple URL typesetting
\usepackage{nicefrac}       % compact symbols for 1/2, etc.
\usepackage{float}
\usepackage{microtype}
\usepackage{graphicx}
\usepackage{subcaption}
\usepackage{xcolor}
\usepackage{booktabs} % for professional tables
\usepackage{lipsum}
\usepackage{xspace}
\usepackage{enumitem}
\usepackage{natbib}
\setcitestyle{numbers}
\setcitestyle{square}
\usepackage{booktabs} % Required for \toprule, \bottomrule
\usepackage{wrapfig}
\usepackage{makecell}

\usepackage{algorithm}
\usepackage{algpseudocode}
\def \header {\vspace{0.1in}\noindent\textbf}

\usepackage[pagebackref=true,breaklinks=true,colorlinks,bookmarks=false,citecolor=blue,linkcolor=blue]{hyperref}
\usepackage{enumitem}
\usepackage{doi}

\usepackage[most]{tcolorbox}

\usepackage{fourier}

\usepackage[textsize=tiny]{todonotes}
\usepackage{xurl}
\usepackage{colortbl}
\usepackage{amsmath}
\usepackage{amssymb}
\usepackage{mathtools}
\usepackage{amsthm}
\newtheorem*{theorem*}{Theorem}

\usepackage[capitalize,noabbrev]{cleveref}

\usepackage{fourier}

\usepackage{amsmath,amsfonts,bm}

\def \header#1{\noindent {\bf #1.}}
\def\eqref#1{equation~\ref{#1}}
\def\1{\bm{1}}

\DeclareMathAlphabet{\mathsfit}{\encodingdefault}{\sfdefault}{m}{sl}
\SetMathAlphabet{\mathsfit}{bold}{\encodingdefault}{\sfdefault}{bx}{n}

\newtheorem{lemma}{Lemma}
\newtheorem{theorem}{Theorem}

\DeclareMathOperator*{\argmin}{arg\,min}

\title{
Empirical Perturbation Analysis of Linear System Solvers from a Data Poisoning Perspective
}

\author{
\textbf{Yixin Liu, Arielle Carr, Lichao Sun}\\
Lehigh University \\
\texttt{\{yila22,arg318,lis221\}@lehigh.edu} \\
}

\renewcommand{\shorttitle}{
Empirical Perturbation Analysis of Linear Solvers from a Data Poisoning Perspective
}

\begin{document}
\maketitle

\begin{abstract}
    The perturbation analysis of linear solvers applied to systems arising broadly in machine learning settings -- for instance, when using linear regression models -- establishes an important perspective when reframing these analyses through the lens of a data poisoning attack. By analyzing solvers' responses to such attacks, this work aims to contribute to the development of more robust linear solvers and provide insights into poisoning attacks on linear solvers. In particular, we investigate how the errors in the input data will affect the fitting error and accuracy of the solution from a linear system-solving algorithm under perturbations common in adversarial attacks. We propose data perturbation through two distinct knowledge levels, developing a poisoning optimization and studying two methods of perturbation: \textbf{Label-guided Perturbation (LP)} and \textbf{Unconditioning Perturbation (UP)}. \textit{Existing works mainly focus on deriving the worst-case perturbation bound from a theoretical perspective, and the analysis is often limited to specific kinds of linear system solvers.} Under the circumstance that the data is intentionally perturbed -- as is the case with data poisoning -- we seek to understand how different kinds of solvers react to these perturbations, identifying those algorithms most impacted by different types of adversarial attacks. We conduct an empirical analysis of the perturbation on linear systems and solvers, covering both direct and iterative solvers, and demonstrate the different sensitivity of each algorithm in the presence of the two proposed perturbations. We develop a theory to describe the convergence of some of these solvers under the data poisoning perturbations, and our results reveal that UP negatively impacts direct solvers more, and LP negatively impacts iterative solvers more.  Moreover, we found that both UP and LP can cause significant convergence slowdowns for most iterative solvers. 
      % We hope that our work contribute to developing more robust linear solvers and providing insights in poisoning attack on linear solvers.
    % This work aims to contribute to the development of more robust linear solvers and provide insights into poisoning attacks on linear solvers, advancing the field of numerical computation and data security.
    % This research aims to provide valuable insights for developing more 
  %  Future research will focus on perturbation analysis of larger linear systems and more advanced solvers, such as the Generalized Minimal Residual Method and Multigrid Method, as well as investigating how these perturbations affect the convergence rates of various solvers.
  \end{abstract}

\section{Introduction}
\label{sec.introd}

In the rapidly evolving landscape of computational science and machine learning, the robustness and reliability of linear system solvers have become increasingly critical \citep{hamon2020robustness}. These solvers are the backbone of numerous applications in data science and engineering \citep{rugh1996linear}. Among the many challenges in efficiently and accurately applying these solvers, one particularly intriguing aspect is the impact of data perturbations on their convergence rates, accuracy, and stability, referred to as perturbation analysis (see, e.g., \citep{carr2023analysis, sifuentes2013gmres} for such analyses specific to GMRES \citep{saad1986gmres}). In most cases, we consider perturbation analysis in the context of ill-conditioned matrices, or noisy data.  Here, we consider it through the lens of data poisoning \citep{schwarzschild2021just}.

Data poisoning is the intentional perturbation of data by malicious attackers \citep{le_etal2023antidreambooth,cina2021hammer,wedin1973perturbation}, which in essence, involves the intentional manipulation of input data to degrade the performance of algorithms. This phenomenon is not merely a theoretical concern but has practical implications in scenarios where data integrity is compromised, either due to malicious intent or inadvertent errors \citep{su2009perturbation}. Understanding and mitigating the effects of such perturbations is thus vital for ensuring the reliability of linear system solvers \citep{rall2006perturbation,kato2012short,kato2013perturbation}. 

Existing works \citep{carr2023analysis, sifuentes2013gmres} on perturbation analysis theory mainly focus on deriving the worst-case perturbation bound from a theoretical perspective, and the analysis is often limited to specific iterative solvers. Under the circumstance that the data is intentionally perturbed, it is important to understand what kinds of perturbation attacks can be conducted and how solvers react to perturbations of different types, paving the way for the design of more robust solvers in these types of applications. From a data poisoning perspective, we conduct an empirical analysis of two types of perturbation in the form of data poisoning attacks on linear systems covering both direct and iterative linear system solvers. We aim to demonstrate how each algorithm behaves under these perturbations.

% In this paper, we focus on empirical perturbation analysis, a method that aims to evaluate the impact of data perturbations on linear system solvers in a practical setting. 
Our goal is to empirically investigate how these perturbations specific to data poisoning attacks influence the accuracy and performance of different solvers. By doing so, we aim to shed light on the resilience of these algorithms under realistic conditions where data may not be pristine, or where it is maliciously perturbed. This study is particularly relevant in the context of machine learning and data analysis, where the integrity of data is paramount, so adversarial attacks must be considered. The insights gained from this research will not only contribute to the enhancement of border learning algorithms \citep{huang2021unlearnable, fu2022robust} that are built on linear solvers but also aid in the development of strategies to counteract the adverse effects of data perturbations \citep{tao2021better}. Our main contributions and findings are as follows:

\begin{enumerate}

% \item We established the forward and backward error bounds for linear system solvers under the $\ell_p$-norm perturbation, i.e., $\| \Delta X \|_p \leq \epsilon $, where $\epsilon$ is the perturbation budget.

\item We reformulate data poisoning as an optimization problem and apply two perturbation strategies, i.e., Label-guided Perturbation (LP) and Unconditioning Perturbation (UP), to degrade the usability of the perturbed data. We also develop convergence theory for the gradient descent solver under LP and UP.
\item We conduct extensive experiments on several direct and iterative linear system solvers to verify the effectiveness of LP and UP. Moreover, we empirically verify the forward error with hypothesis tests.
\item We show that UP is more effective in degrading the usability of the perturbed data for training in the direct solver, while the LP is more effective in negatively impacting the iterative solvers. Moreover, we show that most of the iterative solvers are impacted by the perturbation, and the convergence slows down.
\end{enumerate}

\section{Related Works}
\header{Data Poisoning Attack} \citet{cina2021hammer} propose an efficient heuristic poisoning attack to fool linear classifier models by optimizing the composing coefficients of poison samples for maximizing the estimated likelihood. The evaluations are conducted on the linear support vector machine (SVM) and the logistic regression classifier. \citet{jagielski2018manipulating} study the poisoning attack to linear regression models. Four models are considered in their work, including ordinary least squares (OLS), ridge regression, lasso regression, and elastic net regression. They propose a theoretical-grounded gradient-based poisoning attack and a faster statistical-based poisoning attack based on the observation that the most effective poisons are around the corner of the training data distribution. Targeting deep neural networks, \citet{feng2019learning} propose an alternative updating framework that iteratively updates the perturbation and the surrogate model, which works on high-dimensional image data and can fool many DNNs. \textit{Different from these works, we aim to investigate the poisoning attack to a more fundamental linear system and solvers, which will provide valuable insights for solving the poisoning attack on non-linear models.}

{
\header{Label-guided and Unsupervised Perturbation in Poisoning Attack} In the field of crafting perturbation against neural-network-based classifiers \citep{Liu_2024_CVPR,ren2022transferable,huang2021unlearnable, liu24sem,Ye_2024_CVPR,huang2020metapoison,fu2022robust,feng2019learning}, researchers propose different attacking strategies to efficiently and effectively degrade the performance of the model for both supervised and unsupervised learning setting. We categorize them into two types in this work, i.e., label-guided and unsupervised poisoning attacks. In the label-guided poisoning attack, with label information on downstream evaluations, the attackers craft the poison samples using neither error-maximizing \citep{szegedy2013intriguing,goodfellow2014explaining} or error-minimizing \citep{huang2021unlearnable} strategies. In the unsupervised poisoning attack, the attacker does not have a ground truth label and must craft an effective poison using heuristic strategies based on unsupervised learning tasks. For example, \citet{chen2024one} leverages the text-image latent concept misalignment signal to craft label-agnostic poisoning attacks, which demonstrate better transferability across supervised training with different labels. Furthermore, \citet{wang2024efficient} proposes augmented contrastive poisoning that works on both supervised and unsupervised learning. \textit{In this work, we adapt these two existing poisoning approaches to the linear system and study the sensitivity and accuracy of the solution.}
}

\header{Perturbation Analysis for Linear Solvers} The perturbation analysis of a linear system and its solver is fundamental to numerical linear algebra \citep{kato2013perturbation,thompson1968non}. In this work, we seek to investigate how the errors in the input data will affect the fitting error and accuracy of the solution from a linear system-solving algorithm. 
\citet{su2009perturbation} dived into the relationship of sensitivity, backward errors, and condition numbers in linear systems. \citet{zhou2003perturbation} studied the perturbation of singular linear systems in the generalized linear least squares problem. \citet{sifuentes2013gmres} studied the convergence of GMRES when the coefficient matrix is perturbed with spectral perturbation theory. Targeting the same algorithm, \citet{carr2023analysis} studied the perturbation analysis under a setting of low-rank and small-norm perturbation of the identified coefficient matrix. However, these works all mainly focus on deriving the worst-case perturbation bound from a theoretical perspective, and the analysis is often limited to a specific linear solver. Under the circumstance that the data is intentionally perturbed, it is important to understand how different kinds of solvers react to perturbations of different types, paving the way for designing more robust solvers.
\textit{To bridge the gap, in this work, we conduct the first empirical analysis of the perturbation on linear systems and solvers to show the sensitivities of each algorithm in the presence of two adversarial attack-type perturbations with different knowledge.}

\section{Preliminaries}

\subsection{Linear Systems and Solvers}

A linear system is an equation of the form \(Ax = b\), where \(A \in \mathbb{R}^{n \times d}\) is a known matrix (which can be square or rectangular), \(b \in \mathbb{R}^n\) is a known vector, and \(x \in \mathbb{R}^d\) is the vector of unknowns to be determined. Linear systems are ubiquitous in numerous scientific and engineering disciplines, modeling phenomena in electrical circuits, structural analysis, fluid dynamics, and more. They also play a critical role in computational methods for machine learning and data science, such as in training linear models and as subroutines in nonlinear optimization algorithms.

Efficiently solving linear systems is thus of paramount importance. Many effective solvers have been developed, broadly classified into direct and iterative methods \citep{saad2003iterative, trefethen1997numerical}. Direct methods aim to find an exact solution (within numerical precision limits) in a finite number of steps by performing operations such as matrix factorizations. These methods are reliable for small to medium-sized systems but can become computationally intensive for larger problems. Iterative methods, on the other hand, start with an initial guess and refine the solution through successive approximations \citep{barrett1994templates,saad2003iterative}. They are particularly advantageous for large-scale and sparse systems.

In the context of our study on data poisoning effects, both types of solvers are pertinent. Direct solvers allow us to examine how perturbations in the data influence the exact solutions and the sensitivity of the system. Iterative solvers enable us to investigate the impact of poisoning on the convergence behavior and accuracy of approximate solutions in practical, large-scale scenarios. A comprehensive examination of both provides a thorough understanding of the robustness of linear system solutions under adversarial conditions.

\subsection{Overview of Direct and Iterative Solvers}
% We provide an overview of several commonly used direct and iterative solvers, highlighting their key features and applicability.

\subsubsection{Direct Solvers}

\header{Normal Equations Solver} The Normal Equations Solver (NES) is often used in linear regression to find the least squares solution to an overdetermined system \(Ax = b\), where \(A \in \mathbb{R}^{n \times d}\) with \(n \geq d\). The goal is to minimize the residual norm \(\|Ax - b\|_2\), leading to the normal equations, $A^\top A x = A^\top b$, whose solution is then given by:

\[
x = (A^\top A)^{-1} A^\top b.
\]

This method is straightforward and does not require iterative refinement. However, computing \(A^\top A\) can exacerbate the condition number of \(A\), leading to numerical instability, especially when \(A\) is ill-conditioned. Additionally, forming and inverting \(A^\top A\) can be computationally expensive for large \(d\).

\subsubsection{Iterative Solvers}
\header{Gradient Descent} Gradient Descent solves optimization problems by iteratively moving in the direction of the steepest descent of the objective function $f(x)$. The update rule is:

\[
x^{(k+1)} = x^{(k)} - \alpha_k \nabla f(x^{(k)})
\]

where $\alpha_k$ is the step size and $\nabla f(x^{(k)})$ is the gradient of $f$ at $x^{(k)}$. While this method is simple to implement, it can be slow to converge for ill-conditioned problems. Step size choice is crucial for convergence and efficiency. For the linear system $Ax = b$, we define the objective function $f(x) = \frac{1}{2}\|Ax - b\|_2^2$ with gradient $\nabla f(x) = 2A^T(Ax - b)$.

\header{Jacobi Solver} The Jacobi Method is an iterative algorithm used for solving large, sparse, diagonally dominant linear systems. It updates each component of the solution vector independently using:

\[
x_i^{(k+1)} = \frac{1}{a_{ii}}\left(b_i - \sum_{j \neq i} a_{ij} x_j^{(k)}\right),
\]

where \(x_i^{(k)}\) is the \(i\)th component of the solution at iteration \(k\), and \(a_{ij}\) are the elements of \(A\). Its simplicity and inherent parallelism make it attractive, but convergence can be slow, and it requires \(A\) to be diagonally dominant or positive definite for guaranteed convergence.

\header{Gauss-Seidel } The Gauss-Seidel Method improves upon the Jacobi Method by using the most recent updates within each iteration:

\[
x_i^{(k+1)} = \frac{1}{a_{ii}}\left(b_i - \sum_{j=1}^{i-1} a_{ij} x_j^{(k+1)} - \sum_{j=i+1}^{n} a_{ij} x_j^{(k)}\right).
\]

This often leads to faster convergence compared to the Jacobi Method. However, the method is sequential in nature, making parallel implementation more challenging.

\header{Successive Over-Relaxation (SOR)} The SOR Method introduces a relaxation factor \(\omega\) to accelerate convergence:

\[
x_i^{(k+1)} = (1 - \omega) x_i^{(k)} + \frac{\omega}{a_{ii}}\left(b_i - \sum_{j=1}^{i-1} a_{ij} x_j^{(k+1)} - \sum_{j=i+1}^{n} a_{ij} x_j^{(k)}\right).
\]

By appropriately choosing \(\omega\) (typically \(0 < \omega < 2\)), the method can converge faster than both Jacobi and Gauss-Seidel methods. The optimal \(\omega\) depends on the properties of \(A\) and may require empirical tuning.

\header{Conjugate Gradient (CG)} The CG Method is an efficient iterative solver for large, sparse, symmetric positive-definite systems \citep{hestenes1952methods}. It minimizes the quadratic form associated with the system by generating a sequence of conjugate directions. It updates the solution, residual, and search direction vectors in each iteration:

\[
\begin{aligned}
&x^{(k+1)} = x^{(k)} + \alpha_k p^{(k)}, \\
&r^{(k+1)} = r^{(k)} - \alpha_k A p^{(k)}, \\
&p^{(k+1)} = r^{(k+1)} + \beta_k p^{(k)},
\end{aligned}
\]

where \(\alpha_k\) and \(\beta_k\) are computed using inner products of the vectors. The method converges in at most \(n\) iterations for an \(n \times n\) system, though typically much sooner. Its low memory requirements and efficiency make it suitable for high-dimensional problems.

\header{Generalized Minimal Residual (GMRES) Method} The GMRES Method is designed for solving large, sparse, nonsymmetric linear systems \citep{saad1986gmres}. It constructs an orthonormal basis of the Krylov subspace using the Arnoldi process and seeks to minimize the residual over this subspace. This method computes, $x^{(k)} \in x^{(0)} + \mathcal{K}_k(A, r^{(0)})$ such that $\|b - A x^{(k)}\|_2$ is minimized. Here, \(\mathcal{K}_k(A, r^{(0)}) = \text{span}\{r^{(0)}, A r^{(0)}, A^2 r^{(0)}, \dots, A^{k-1} r^{(0)}\}\) is defined as the Krylov space of dimension $k$. The method is robust and can handle a wide range of systems, but its computational cost and memory usage increase with each iteration. Restarting strategies (e.g., GMRES(m)) are commonly employed to limit these costs by restarting the method after every \(m\) iterations.
\section{Problem Statement and Threat Model}
\header{Linear System Setting}
Under common settings of machine learning, we denote the data features as $X \in \mathbb{R}^{n \times d}$ and the labels as $y \in \mathbb{R}^n$, sampling from the data distribution $\mathcal{D}$.
A linear regression model $f$ is defined as $f(X;w) =Xw$, where $w \in \mathbb{R}^d$ is the weight vector. Given a set of training data $(X_{\text{train}}, y_{\text{train}})$, the goal of the linear regression model is to solve the linear system $X_{\text{train}} w = y_{\text{train}}$ and obtain the optimal weight vector $w^*$. We hope that the obtained weight vector $w^*$ can generalize well on the data sampled from the same distribution, i.e., when the generalization error$L_{\text{gen}}(w^*) = \mathbb{E}_{(X,y) \sim \mathcal{D}} [L(w^*)]$, where $L(w^*) = \frac{1}{n} \|Xw^* - y\|_2^2$ is low. 

\header{Attacker's Goal and Capability}
When the training data is poisoned by a malicious attacker, the obtained weight vector $w^*$ might have a large generalization error. 
In the data poisoning setting, an attacker can perturb the data feature to degrade the generalization ability of the trained model. We assume that only the feature will be perturbed, with a $\ell_p$-norm constraint ball with budget $\epsilon$, i.e., $||X - X^\prime||_p \le \epsilon$, where $X^\prime$ is the perturbed data feature.
% \newpage 
Under this setting, we study the following two important questions on the perturbation analysis of the linear system.

\begin{enumerate}
    \item Given data $(X, y)$, and the perturbation of the features from $X$ to $X^\prime = X + \Delta X$ within some perturbation budget $\epsilon$, 
     can we bound the relative change of the obtained solution $\frac{||w-w^\prime||}{||w||}$? Can we bound the system output error $||Xw - X w^\prime||$ on the training set? If the attacker seeks to perturb the new solution to be $w^\prime$, what is the smallest perturbation we can find on the given space? 
    \item With the different levels of knowledge (the access to the testing set), how can we conduct perturbation that maximizes the generalization error of the fitted model? 
    How does it perform when we directly maximize the condition number of the perturbed matrix $X^\prime$? 
    This question aims to study the feasibility of degrading the usability of the perturbed data for training in practice.
\end{enumerate}

\section{Methodology}

\subsection{Matrix Perturbation as An Optimization Problem}
\label{sec.matrix_perturbation}

Assuming the testing data $(X_t, y_t) \sim  \mathcal{D}$, we aim to find a perturbation $\Delta X$ to the matrix $X$ such that when a linear solver obtains the solution $w^\prime$ from the perturbed data $(X + \Delta X, y)$, the system output error $||y_t - X_t w^\prime||$ on the testing set is maximized. We formulate the optimization problem as follows,
\begin{equation}
  \begin{aligned}
    \max_{\Delta X} \quad & ||y_t - X_t w^\prime|| \quad  \text{s.t.} ||\Delta X|| \le \epsilon .
  \end{aligned}
  \label{eq:opt}
\end{equation}

{Tackling a similar objective as Eq. \ref{eq:opt}, clean-label poisoning attacks \citep{fu2022robust,huang2021unlearnable} in classification learning aims to perturb the feature of data to mislead the learning of model, maximizing the generalization error of learned model on the testing set. Based on two kinds of knowledge setting, i.e., whether the attacker can access the label of data, these works can be categorized into two types: \textit{label-guided} poisoning attacks \citep{huang2021unlearnable,fu2022robust}, and \textit{label-free} poisoning attacks \citep{chen2024one,ren2022transferable,zhu2019transferable,Zhang_2023_CVPR,heindiscriminate}. Label-guided poisoning attacks craft perturbation by minimizing \citep{huang2021unlearnable} or maximizing \citep{goodfellow2014explaining} the training cross-entropy loss with the guidance of the label, aiming to craft perturbation that tricks the model into learning the false correlations. Advanced label-guided poisoning attacks further craft class-wise perturbation by targeting adversarial perturbation \citep{fowl2021adversarial}. On the other hand, under the setting that no label information is available beforehand, label-free poisoning attacks leverage heuristic methods to craft perturbation, such as multi-modal misalignment \citep{chen2024one}, contrastive loss \citep{zhu2019transferable,zhu2019transferable}, and augmented transformation loss \citep{wang2024efficient}. By perturbing the data from its semantic meaning \citep{yu2022availability}, attackers can prevent the modified data from being used for learning effective representation or any downstream classification problem with different label configurations. 
}

In this work, we adapt these two important perturbation paradigms to the problem of poisoning linear solvers under different knowledge settings. First, for the setting of label-guided poisoning attacks, we can directly solve the optimization problem in Eq. \ref{eq:opt}. We term this approach as \textit{Label-guided Perturbation (LP)}. However, when the attacker can not directly access the testing data $(X_t, y_t)$, we can only access the training data. Under this challenging setting, motivating that the condition number of the matrix $X$ is a good indicator of the ill-conditioning of the linear system, we propose to maximize the condition number of the perturbed matrix $X^\prime$ to degrade the usability of the perturbed data. We term this approach that solely maximizes the condition number of the perturbed matrix as \textit{Unconditioning Perturbation (UP)}. Compared to the LP, the UP is more general and requires fewer assumptions on the label availability, facilitating its application in both supervised and unsupervised learning settings \citep{chen2024one, Zhang_2023_CVPR}.
For LP and UP, we solve the following optimization problems respectively:
\begin{equation}\label{eq:lp_up_opt}
  \begin{aligned}
    \text{LP:} \quad \max_{\Delta X} \quad & ||y_t - X_t w^\prime|| \quad \text{s.t.} \quad ||\Delta X|| \le \epsilon, \quad w^\prime = \argmin_{w} ||y - (X+\Delta X) w|| ,\\
    \text{UP:} \quad \max_{\Delta X} \quad & \kappa(X + \Delta X) \quad \text{s.t.} \quad ||\Delta X|| \le \epsilon, 
  \end{aligned}
  % \label{eq:lp_up_opt}
\end{equation}

where $w^\prime$ is the solution from the perturbed system for LP, both problems can be solved using gradient-based optimization methods with appropriate projections to satisfy the perturbation constraint.

\subsection{Theoretical Analysis of UP and LP with Gradient Descent Solver}
In this section, we theoretically analyze the impact of UP and LP on the convergence and solution of linear systems solved using gradient descent solver, which is a widely used iterative solver in practice \citep{huang2021unlearnable,goodfellow2014explaining,Liu_2024_CVPR}. Future work will consider similar theories for other popular solvers, such as GMRES.

\begin{theorem}[\textbf{Convergence Rate of $\alpha$-UP for L-smooth and Convex Functions}]
  Let $f(w)$ be an objective function that is L-smooth and convex, e.g., $f(w; X) = \frac{1}{2} \| X w - y \|^2$, where $X$ is a data matrix. Denote the function after applying UP as $f{^\prime}(w; X+ \Delta X)$, which is also L-smooth and convex. Assume that UP is applied, increasing the condition number of the system such that the maximum singular value of the perturbed matrix is $\sigma_{\max}(X + \Delta X) = \alpha \sigma_{\max}(X)$, where $1 < \alpha < \sqrt{2/\gamma}$. Let $w^*$ be the optimal solution, and $w_T^\prime$ be the solution after $T^\prime$ iterations of gradient descent with step size $\gamma$. Then:
  
  \begin{enumerate}
      \item \textbf{Convergence Rate}: After applying UP, for an L-smooth and convex function, we have:
      \[
      \min_{0 \leq t \leq T-1} \left( f^\prime(w_t) - f^\prime(w^*) \right) \leq  \frac{C}{\gamma (2 - \alpha ^2 )} \frac{1}{T^\prime},
      \]
      where $\| w_0 - w^* \|^2 \leq C$ represents the initial distance between $w_0$ and the optimal solution $w^*$ is bounded by constant $C > 0$.
  
      \item \textbf{Total Iterations for $\beta$-accuracy}: The total number of iterations $T^\prime$ required to obtain a solution such that $f^\prime(w_T) - f^\prime(w^*) \leq \beta$ is:
      \[
      T^\prime \geq \frac{C}{\gamma (2 - \alpha ^2 ) } \cdot \frac{1}{\beta} .
      \]
  \end{enumerate}
\end{theorem}

See the proof in the Appendix \ref{app:theorem_alpha_up_proof}. This theorem demonstrates that UP impacts the smoothness factor $L$ of the loss function $f(w)$ by a factor of $\alpha^2$, i.e., $L^\prime = \alpha^2 L$ under the assumption of change of the maximum singular value of the matrix $X$ by a factor of $\alpha$. 

\begin{theorem}[\textbf{Lower Bound on Solution Divergence of $\eta$-LP}]
Let $w^*$ be the optimal solution for $X w = y$, and $w'^*$ be the optimal solution for the perturbed system $(X + \Delta X) w'^* = y + \Delta y$. Assume $\| \Delta X \|_2 \leq \epsilon$, and the perturbation induces a significant change in the prediction such that $\| \Delta y \|_2 \geq \eta$, where $\eta > 0$ is a threshold capturing the minimum size of the target change. Then, the difference between $w^*$ and $w'^*$ is lower-bounded by:

\[
\| w^* - w'^* \| \geq \frac{\eta}{\| X \|} \cdot \frac{1}{1 + \epsilon \| X^{-1} \|}.
\]

\end{theorem}

See the proof in the Appendix \ref{app:theorem_eta_lp_proof}. This theorem demonstrates that perturbations like LP force the solution to diverge from the original optimal solution by at least an amount proportional to the induced target change $\eta$, the condition number of the system $X$, and the perturbation size $\epsilon$. The divergence increases as the perturbation size grows, the condition number worsens, or the induced change in prediction becomes larger.

\begin{figure}[thbp]
    \centering
    \begin{subfigure}[b]{0.32\textwidth}
        \centering
        \includegraphics[width=\textwidth]{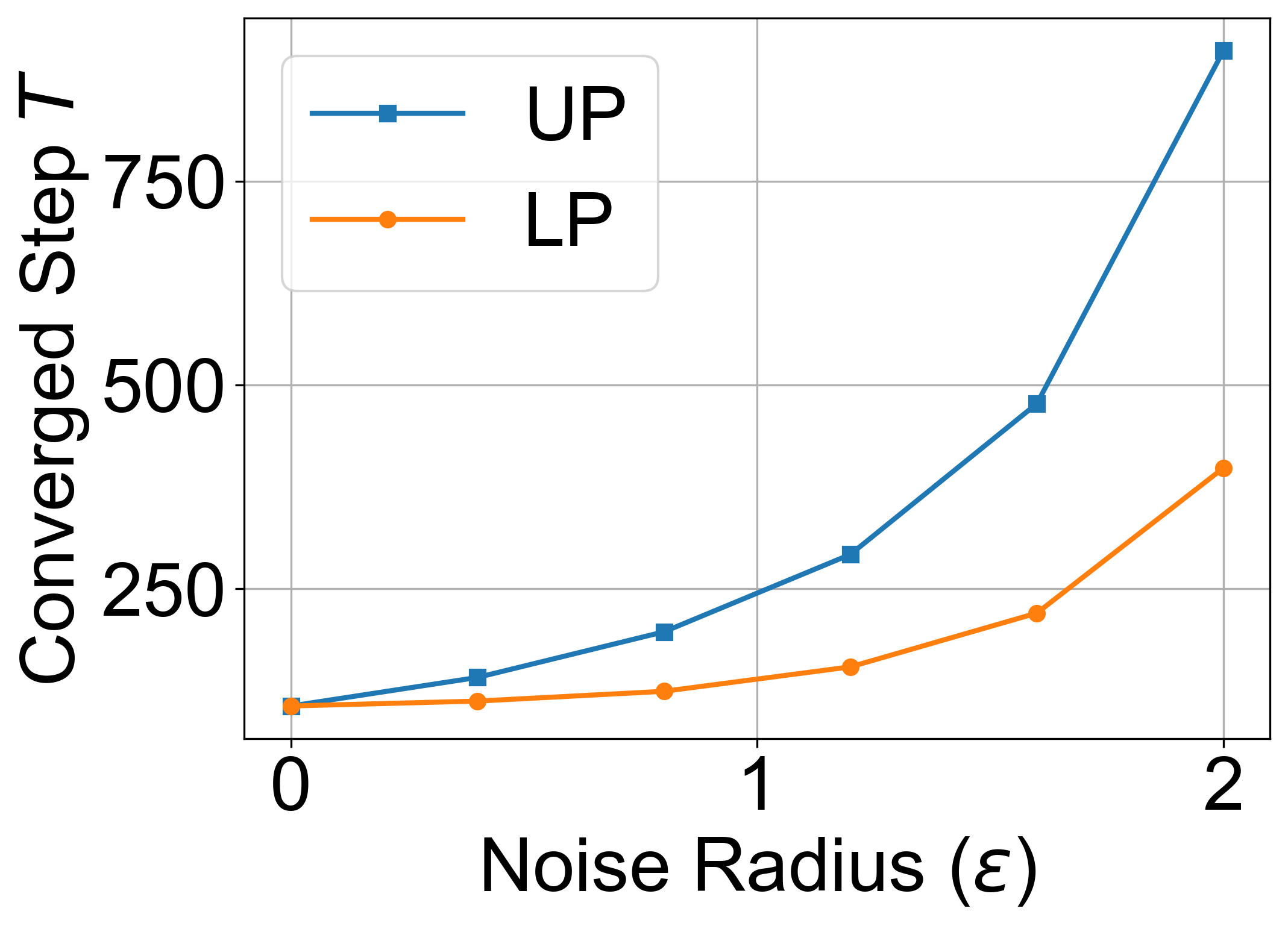}
        \caption{GD steps with UP}
    \end{subfigure}
    \hfill
    \begin{subfigure}[b]{0.32\textwidth}
      \centering
      \includegraphics[width=\textwidth]{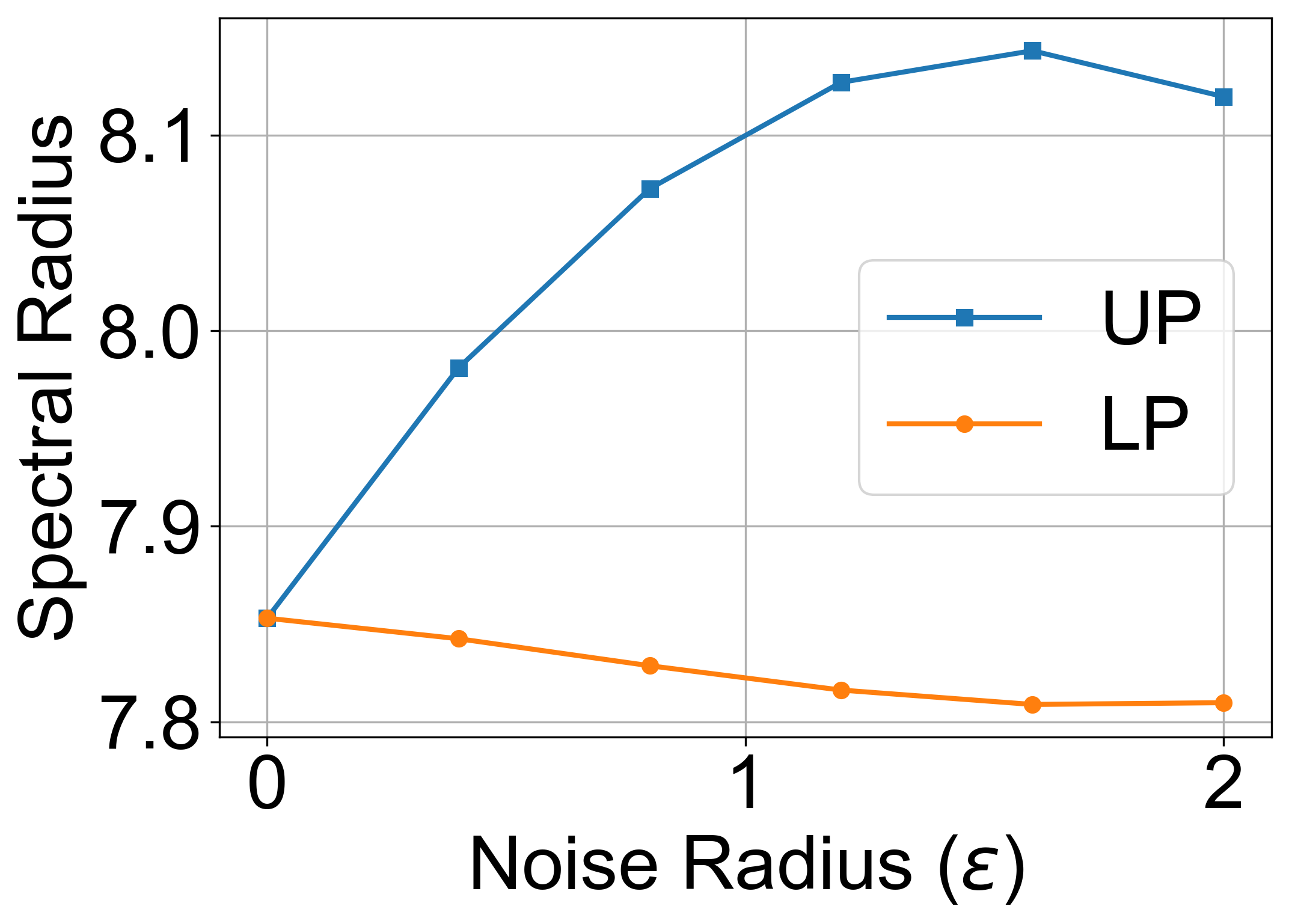}
      \caption{GD spectral radii of $X^{\prime}$}
    \end{subfigure}
    \hfill
    \begin{subfigure}[b]{0.32\textwidth}
        \centering
        \includegraphics[width=\textwidth]{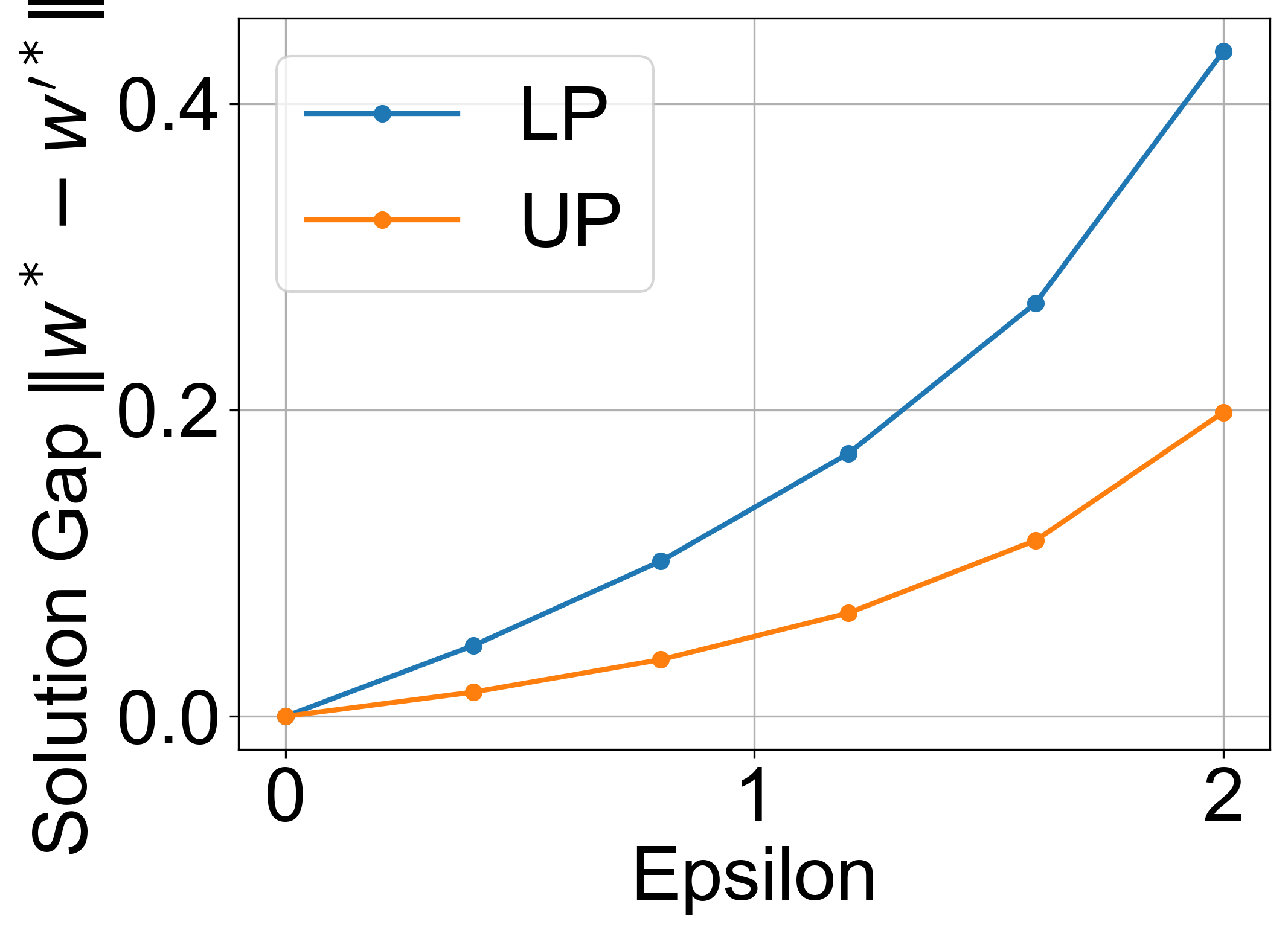}
        \caption{GD solution gap with LP}
    \end{subfigure}
    \caption{Empirical testing of convergence analysis of LP and UP on Gradient Descent (GD). (a) shows the impact of UP on GD convergence steps across noise radii, (b) displays the spectral radii of perturbed $X^\prime$ across noise radii, and (c) shows the solution gap between the perturbed solution and the original optimal solution across noise radii.}
    \label{fig:gd_convergence}
\end{figure}

\subsection{General Forward $\ell_p$-norm Perturbation Error Bound}
{
We further provide a theorem on the forward error bound, which establishes an upper limit on the relative change in the solution and the system output error when the input data is perturbed within a specified $\ell_p$-norm constraint. This theorem offers valuable insights into the sensitivity of the linear system to input perturbations and can be particularly useful in assessing the potential impact of poisoning attacks.
}

\begin{theorem}[\textbf{$\ell_p$ Forward Error Bound}]
Given a constraint $||\Delta X || \leq \epsilon $, $(X+\Delta X) w^\prime = y$ and $X w = y$, assuming $\epsilon ||A^{-1}|| < 1$, we have
\begin{equation}
  \frac{\left\|w-w^{\prime}\right\|}{\|w\|} \leq 
  \frac{\epsilon \left\|X^{-1}\right\|}{1-\epsilon \left\|X^{-1}\right\|}. 
\end{equation}
  We can bound the system output error with the following equation,
  \begin{equation}
    \begin{aligned}
      \left\|X w^{\prime}-X w\right\| \leq \frac{\epsilon \|w\| \kappa(X) }{1-\epsilon\left\|X^{-1}\right\|}.
    \end{aligned}  
  \end{equation}
\end{theorem}

{
For a detailed proof, see Saad's textbook \citep{saad2003iterative}. This theorem provides insights into linear systems' sensitivity to input perturbations, which is relevant to data poisoning attacks. We empirically verify this forward error bound in Section \ref{sec.direct_solver}, demonstrating how direct solvers respond to such perturbations.
}

\section{Experiments}
\subsection{Experimental Setup}
Following \citet{russo2023analysis}, we use Sequential Least Squares Programming (SLSQP) \citep{kraft1988software} to solve the perturbation optimization problem in \eqref{eq:lp_up_opt}. We define the constraint on the matrix $\ell_p$-norm of $\Delta X$ as the perturbation budget $\epsilon$ following poisoning setting \citep{huang2020unlearnable}, i.e., $\Vert \Delta X \Vert_p \leq \epsilon$. We set $\epsilon = 0.1$ and $p = 2$ by default, the maximum iteration number to 1000, and the tolerance value as $1 \times {10}^{-10}$. 

In the data synthesis process, \textit{for the direct solver}, we use the {\tt make\_regression} in the {\tt sklearn} library \citep{pedregosa2011scikit}. By default, we set the feature dimension as $d=3$ and $n=6$ for the training set and $d=3$ and $n=9$ for the testing set. For the data synthesis of $X$ in \textit{the iterative solver}, we use the {\tt sparse\_random} in the {\tt scipy} library \citep{2020SciPy-NMeth} with additional steps to make $X$ symmetric and diagonally dominant (with $n=d=20$ by default). After generating the training and testing matrix $X$ and $X_t$, we generate the training label $y$ by sampling from the uniform distribution $U(0, 1)^{n}$. The testing label $y_t$ is set as the dot product of $X_t$ and the fitted solution $w$ using the least square method for the original system $X w = y$, i.e., $y_t = X_t w$. We consider the following six commonly used iterative solvers: Jacobi, Gauss-Seidel, Successive Over-Relaxation (SOR), Conjugate Gradient (CG), Generalized Minimal Residual (GMRES), and Gradient Descent (GD). We set the perturbation radii $\epsilon$ from $0.4$ to $2.0$ with step size $0.2$. 

To quantify the solver's accuracy and sensitivity upon perturbation, we use the following metrics: 
i) \textit{Absolute error $\| y_t - X_t w^\prime\|$}: This metric reflects the overall deviation of the predicted output from the true output, providing a direct measure of the accuracy of the solver.
ii) \textit{Relative testing residual} ${\tt rsd}(w)=\| y_t - X_t w^\prime\| / \| y_t \|$: This metric normalizes the absolute error by the norm of the true output, offering a perspective on the error relative to the magnitude of the true output.
iii) \textit{Relative solution error} ${\tt err}(w)=\| w - w^\prime\| / \| w \|$: This metric measures the deviation of the perturbed solution from the original solution, indicating the sensitivity of the solution to perturbations.
iv) \textit{Condition number} $\kappa(X)$: This metric provides insight into the numerical stability of the matrix $X$, with higher values indicating greater sensitivity to perturbations.
v) \textit{Final converged iteration}: the iteration number when the solution is converged: This metric reflects the efficiency of the solver in terms of the number of iterations required to reach convergence.
% \begin{itemize}
%     \item Absolute error: $\| y_t - X_t w^\prime\|$
%     \item Relative testing residual: ${\tt rsd}(w)=\| y_t - X_t w^\prime\| / \| y_t \|$
%     \item Relative solution error: ${\tt err}(w)=\| w - w^\prime\| / \| w \|$
%     \item Condition number: $\kappa(X)$
%     \item Final converged iteration: the iteration number when the solution is converged.
% \end{itemize}

% relative testing residual ${\tt rsd}(w)=||y_t - X_t w^\prime|| / ||y_t||$ and the relative solution error ${\tt err}(w)=||w - w^\prime|| / ||w||$. 
  
\begin{table}
  \centering
  \caption{Effect of perturbation radii \(\epsilon\) on absolute error, relative error, accuracy, and condition number using the direct solver, NES. We set $d=3$ and $n=6$ for the training set and $d=3$ and $n=9$ for the testing set. }
  \label{tab:epsilon_results_table1}
  \begin{tabular}{ccccccccccc} 
  \toprule
  \multirow{2}{*}{$\epsilon$} & \multicolumn{2}{c}{$\| y_t - X_t w^\prime\|$} & \multicolumn{2}{c}{${\tt rsd}(w)$} & \multicolumn{2}{c}{$||w - w^\prime||$} & \multicolumn{2}{c}{${\tt err}(w)$} & \multicolumn{2}{c}{$\kappa(X)$}  \\
  \cmidrule(lr){2-11} 
                              & LP      & UP                             & LP & UP                            & LP & UP                                & LP & UP                            & LP & UP                          \\ 
  \midrule
  % 0.01  & 0.3691  & 0.8961 & 0.0013 & 0.0032 & 0.1008 & 0.2588 & 0.0011 & 0.0028 & 2.7327 & 2.7531 \\
  % 0.1   & 3.0892  & 9.4630 & 0.0109 & 0.0333 & 0.9508 & 2.7322 & 0.0103 & 0.0297 & 2.7081 & 2.9297 \\
  % 0.5   & 16.7324 & 65.3496& 0.0588 & 0.2298 & 4.0974 & 18.8579& 0.0446 & 0.2051 & 3.1702 & 4.0489 \\
  % 1     & 31.3638 & 246.2708& 0.1103 & 0.8659 & 9.6569 & 71.0407& 0.1050 & 0.7727 & 3.7102 & 7.5093 \\
  0.01 & 1.8298&0.5716&0.0076&0.0024&0.4985&0.1856&0.0051&0.0019&1.9244&1.9395 \\
0.1 & 18.2152&5.8403&0.0752&0.0241&4.931&1.8827&0.0503&0.0192&1.899&2.0515 \\
0.5 & 99.8274&34.6936&0.4123&0.1433&28.3635&10.9775&0.2894&0.112&2.2431&2.7119 \\
1.0 & 74.3361&105.8894&0.3071&0.4374&26.2289&33.7195&0.2676&0.3441&2.1506&4.3493 \\
  \bottomrule
  \end{tabular}
  
  \end{table}

\begin{table}[h]
  \centering
  \caption{Effect of perturbation radii \(\epsilon\) on absolute error, relative error, accuracy, and convergence steps for different iterative solvers. Results of three radii are shown due to space limit.}
  \label{tab:main_iterative_results}
 \resizebox{\linewidth}{!}{
  \begin{tabular}{ cccc cccc cccc}
    \toprule
    \toprule
    \multirow{2}{*}{Solver} & \multicolumn{1}{c}{\multirow{2}{*}{$\epsilon$}} & \multicolumn{2}{c}{$\| y_t - X_t w^\prime\|$}             & \multicolumn{2}{c}{${\tt rsd}(w)$}              & \multicolumn{2}{c}{$||w - w^\prime||$}          & \multicolumn{2}{c}{${\tt err}(w)$}              & \multicolumn{2}{c}{$N_{end}$}                    \\
    \cmidrule(l){3-12}
                            &                             & \multicolumn{1}{c}{LP}      & \multicolumn{1}{c}{UP}      & \multicolumn{1}{c}{LP} & \multicolumn{1}{c}{UP} & \multicolumn{1}{c}{LP} & \multicolumn{1}{c}{UP} & \multicolumn{1}{c}{LP} & \multicolumn{1}{c}{UP} & \multicolumn{1}{c}{LP} & \multicolumn{1}{c}{UP}  \\
    \midrule
    \multirow{3}{*}{Jacobi}
      & 0.0 & 2.08166 & 2.08166 & 0.85064 & 0.85064 & 0 & 0 & 0 & 0 & 32 & 32 \\
      % & 0.4 & 2.31239 & 2.09823 & 0.94493 & 0.85742 & 0.04613 & 0.01577 & 0.10890 & 0.03722 & 32 & 34 \\
      & 0.8 & 2.58736 & 2.12627 & 1.05729 & 0.86887 & 0.10130 & 0.03696 & 0.23916 & 0.08727 & 32 & 37 \\
      % & 1.2 & 2.93095 & 2.17607 & 1.19770 & 0.88922 & 0.17154 & 0.06741 & 0.40500 & 0.15916 & 33 & 44 \\
      % & 1.6 & 3.38871 & 2.27310 & 1.38476 & 0.92887 & 0.26989 & 0.11487 & 0.63721 & 0.27121 & 37 & 58 \\
      & 2.0 & 4.07078 & 2.49233 & 1.66347 & 1.01846 & 0.43424 & 0.19852 & 1.02525 & 0.46871 & 53 & 84 \\
    \midrule
    \multirow{3}{*}{Gauss-Seidel}
      & 0.0 & 2.08166 & 2.08166 & 0.85064 & 0.85064 & 0 & 0 & 0 & 0 & 12 & 12 \\
      % & 0.4 & 2.31239 & 2.09823 & 0.94493 & 0.85742 & 0.04613 & 0.01577 & 0.10890 & 0.03722 & 12 & 14 \\
      & 0.8 & 2.58736 & 2.12627 & 1.05729 & 0.86887 & 0.10130 & 0.03696 & 0.23916 & 0.08727 & 13 & 18 \\
      % & 1.2 & 2.93095 & 2.17607 & 1.19770 & 0.88922 & 0.17154 & 0.06741 & 0.40500 & 0.15916 & 15 & 23 \\
      % & 1.6 & 3.38871 & 2.27310 & 1.38476 & 0.92887 & 0.26989 & 0.11487 & 0.63721 & 0.27121 & 19 & 30 \\
      & 2.0 & 4.07078 & 2.49233 & 1.66347 & 1.01846 & 0.43424 & 0.19852 & 1.02525 & 0.46871 & 28 & 43 \\
    \midrule
    \multirow{3}{*}{SOR}
      & 0.0 & 2.08166 & 2.08166 & 0.85064 & 0.85064 & 0 & 0 & 0 & 0 & 12 & 12 \\
      % & 0.4 & 2.31239 & 2.09823 & 0.94493 & 0.85742 & 0.04613 & 0.01577 & 0.10890 & 0.03722 & 12 & 14 \\
      & 0.8 & 2.58736 & 2.12627 & 1.05729 & 0.86887 & 0.10130 & 0.03696 & 0.23916 & 0.08727 & 13 & 18 \\
      % & 1.2 & 2.93095 & 2.17607 & 1.19770 & 0.88922 & 0.17154 & 0.06741 & 0.40500 & 0.15916 & 15 & 23 \\
      % & 1.6 & 3.38871 & 2.27310 & 1.38476 & 0.92887 & 0.26989 & 0.11487 & 0.63721 & 0.27121 & 19 & 30 \\
      & 2.0 & 4.07078 & 2.49233 & 1.66347 & 1.01846 & 0.43424 & 0.19852 & 1.02525 & 0.46871 & 28 & 43 \\
    \midrule
    \multirow{3}{*}{GMRES}
      & 0.0 & 2.08166 & 2.08166 & 0.85064 & 0.85064 & 0 & 0 & 0 & 0 & 14 & 14 \\
      % & 0.4 & 2.31239 & 2.09823 & 0.94493 & 0.85742 & 0.04613 & 0.01577 & 0.10890 & 0.03722 & 14 & 14 \\
      & 0.8 & 2.58736 & 2.12627 & 1.05729 & 0.86887 & 0.10130 & 0.03696 & 0.23916 & 0.08727 & 14 & 14 \\
      % & 1.2 & 2.93095 & 2.17607 & 1.19770 & 0.88922 & 0.17154 & 0.06741 & 0.40500 & 0.15916 & 14 & 14 \\
      % & 1.6 & 3.38871 & 2.27310 & 1.38476 & 0.92887 & 0.26989 & 0.11487 & 0.63721 & 0.27121 & 14 & 14 \\
      & 2.0 & 4.07078 & 2.49233 & 1.66347 & 1.01846 & 0.43424 & 0.19852 & 1.02525 & 0.46871 & 14 & 15 \\
    \midrule
    \multirow{3}{*}{Conjugate Gradient}
      & 0.0 & 2.08166 & 2.08166 & 0.85064 & 0.85064 & 0 & 0 & 0 & 0 & 14 & 14 \\
      % & 0.4 & 2.31239 & 2.09823 & 0.94493 & 0.85742 & 0.04613 & 0.01577 & 0.10890 & 0.03722 & 16 & 14 \\
      & 0.8 & 2.58736 & 2.12627 & 1.05729 & 0.86887 & 0.10130 & 0.03696 & 0.23916 & 0.08727 & 19 & 14 \\
      % & 1.2 & 2.93095 & 2.17607 & 1.19770 & 0.88922 & 0.17154 & 0.06741 & 0.40500 & 0.15916 & 22 & 14 \\
      % & 1.6 & 3.38871 & 2.27310 & 1.38476 & 0.92887 & 0.26989 & 0.11487 & 0.63721 & 0.27121 & 25 & 16 \\
      & 2.0 & 4.07078 & 2.49233 & 1.66347 & 1.01846 & 0.43424 & 0.19852 & 1.02525 & 0.46871 & 29 & 17 \\
    \midrule
    \multirow{3}{*}{Gradient Descent}
      & 0.0 & 2.08166 & 2.08166 & 0.85064 & 0.85064 & 0 & 0 & 0 & 0 & 106 & 106 \\
      % & 0.4 & 2.31239 & 2.09823 & 0.94493 & 0.85742 & 0.04613 & 0.01577 & 0.10890 & 0.03722 & 112 & 141 \\
      & 0.8 & 2.58736 & 2.12627 & 1.05729 & 0.86887 & 0.10130 & 0.03696 & 0.23916 & 0.08727 & 124 & 197 \\
      % & 1.2 & 2.93095 & 2.17607 & 1.19770 & 0.88922 & 0.17154 & 0.06741 & 0.40500 & 0.15916 & 154 & 292 \\
      % & 1.6 & 3.38871 & 2.27310 & 1.38476 & 0.92887 & 0.26989 & 0.11487 & 0.63721 & 0.27121 & 220 & 477 \\
      & 2.0 & 4.07078 & 2.49233 & 1.66347 & 1.01846 & 0.43424 & 0.19852 & 1.02525 & 0.46871 & 398 & 910 \\
    % \midrule
    \bottomrule
    \bottomrule
    \end{tabular}
 }
\end{table}

\subsection{Results and Analysis on Direct Solver}
\label{sec.direct_solver}
% \header{Effectiveness of LP and UP on Direct Solver} We test out the effect of a wide range of perturbation radii $\epsilon$ on the system output error $||y_t - X_t w^\prime||$ on the testing set. We set the perturbation radii $\epsilon$ as $\{0.01, 0.1, 0.5, 1.0\}$. The results are shown in Table \ref{tab:epsilon_results_table1}. We can see that the system output error increases as the perturbation radii $\epsilon$ increases. This is because the larger perturbation radii $\epsilon$ can lead to larger perturbation $\Delta X$, which can lead to larger system output error. We can also see that the UP is more effective than LP in terms of increasing the error and residual with a larger condition number. This suggests that for the direct solver, the UP is more effective in degrading the perturbed data usability.

\header{Effectiveness of UP over LP on Normal Equations Solver} We evaluate the impact of both Label-guided Perturbation (LP) and Unconditioning Perturbation (UP) on the Normal Equations Solver, a direct method for solving linear systems. We test a range of perturbation radii \(\epsilon\) (\(\{0.01, 0.1, 0.5, 1.0\}\)) and measure the system output error \(\|y_t - X_t w^\prime\|\) on the testing set. The results are presented in Table \ref{tab:epsilon_results_table1}. From the table, we observe that as the perturbation radius \(\epsilon\) increases, the system output error also increases for both LP and UP, which is expected due to the larger perturbations \(\Delta X\). Notably, UP consistently results in a significantly higher error compared to LP across all values of \(\epsilon\). Additionally, the condition number \(\kappa(X)\) of the perturbed matrix increases more under UP than LP. 

\header{Analysis of LP and UP on Normal Equations Solver} 
% This phenomenon can be understood by considering the sensitivity of the Normal Equations Solver to perturbations in the coefficient matrix \(X\). 
We further analyze this phenomenon by considering a sensitivity analysis of NES in the coefficient matrix \(X\). Note that the condition number of \(X^\top X\) is the square of the condition number of \(X\), i.e., \(\kappa(X^\top X) = \kappa(X)^2\) \citep{trefethen1997numerical}. This implies that any increase in \(\kappa(X)\) due to perturbations is significantly amplified when solving the normal equations. The UP strategy specifically aims to maximize the condition number of the perturbed matrix \(X^\prime\), effectively making \(X^\prime\) ill-conditioned. This ill-conditioning leads to a substantial amplification of errors in the solution \(w^\prime\) when using the NES. In contrast, the LP attack focuses on perturbing \(X\) in a way guided by increasing error in the right-hand side, which does not impact the condition number of \(X\). Therefore, UP is more effective than LP in degrading the usability of the perturbed data in the context of the NES. By directly increasing \(\kappa(X)\), UP exploits the solver's inherent sensitivity to ill-conditioned matrices, resulting in greater errors in the solution and, consequently, a higher testing error.

\header{Empirical Verification of the Forward Perturbation Bound} To verify the forward perturbation bound, we conduct experiments on the setting of $n=d=3$ with different perturbation radii $\epsilon \in \{0.01, 0.011, 0.012, 0.013\}$, which satisfy $\epsilon ||A^{-1}|| < 1$ and $\kappa (A)\leq 10 ^2 $. We generate data with 100 different random seeds for each perturbation radius $\epsilon$ and compute the mean value of the relative error and forward upper bound. 
Furthermore, we conduct a one-sided t-test to verify whether the empirical results are consistent with the theoretical bound, with the null hypothesis that the relative solution error is larger than the forward upper bound and significance level $\xi = 0.05$. As shown in Table \ref{tab:forward_results}, we can see that the empirical results are consistent with the theoretical bound under different perturbation radii $\epsilon$. 
% Note that we only verify the budget $\epsilon$ that satisfied the condition $\epsilon ||A^{-1}|| < 1$.

\begin{table}
  \centering
  \caption{Empirical verification of the forward perturbation bound ($n=d=3$) with 100 sampling times. The mean value of the relative error and forward upper bound are shown.}
  \label{tab:forward_results}
  \begin{tabular}{ccccc cccc} 
  \toprule
  \multirow{2}{*}{$\epsilon$} & \multicolumn{4}{c}{LP} & \multicolumn{4}{c}{UP} \\
  \cmidrule(lr){2-5}
  \cmidrule(lr){6-9}
                            & {${\|w - w^\prime\| / \|w\|}$} & { $ \frac{\epsilon \left\|X^{-1}\right\|}{1-\epsilon \left\|X^{-1}\right\|}$} & {T-statistic}&{P-value }     & {${\|w - w^\prime\| / \|w\|}$} & { $ \frac{\epsilon \left\|X^{-1}\right\|}{1-\epsilon \left\|X^{-1}\right\|}$} & {T-statistic}&{P-value }   \\ 
  \midrule
  0.01                        & 0.0403 & 0.06644 & -2.6485 & 0.0052   & 0.0153 & 0.04342 & -3.0533 & 0.00174 \\
  0.011                         &  0.0472 & 0.10121 & -1.9299 & 0.02939 & 0.0224 & 0.05043 & -4.4487 & 2e-05 \\
  0.012                        & 0.0479 & 0.07454 & -3.9185 & 0.00011  & 0.0236 & 0.07855 & -3.0158 & 0.00198 \\
  0.013         &0.0553 & 0.07481 & -3.8578 & 0.00016               & 0.0354 & 0.1578 & -1.8173 & 0.03717 \\
  \bottomrule
  \end{tabular}
  \end{table}

% \header{Visualization of Original Data and Perturbed Data} Under the setting of $n=5, d=3$ and $\epsilon \in \{1, 4, 10\}$, we visualize the original data $X$ and perturbed one $X^\prime$ in Figure \ref{fig:vis-pertub}. We can see with the increase of the perturbation radii $\epsilon$, the perturbed data $X^\prime$ is gradually deviated from the original data $X$. This is consistent with our intuition that the larger perturbation radii $\epsilon$ can lead to larger perturbation $\Delta X$. \revised{this can be dropped due to little information }

% \begin{figure}
%  \begin{minipage}[t]{0.32\linewidth}
%   \centering
%   \includegraphics[width=1.0\textwidth]{fig/vis-plane-1.pdf}
%   \caption{$\epsilon = 1$.}
%   \label{fig:vis-pertub-1}
%  \end{minipage}
%   \begin{minipage}[t]{0.32\linewidth}
%     \centering
%     \includegraphics[width=1.0\textwidth]{fig/vis-plane-4.pdf}
%     \caption{$\epsilon = 4$.}
%     \label{fig:vis-pertub-4}
%     \end{minipage}
%     \begin{minipage}[t]{0.32\linewidth}
%       \centering
%       \includegraphics[width=1.0\textwidth]{fig/vis-plane-10.pdf}
%       \caption{$\epsilon = 10$.}
%       \label{fig:vis-pertub-10}
%     \end{minipage}
%     \caption{Visualization of original data $X$ and perturbed data $X^\prime$ using UP with different perturbation radii $\epsilon$ (with $n=5, d=3$). The cyan plane belongs to clean data, while the yellow plane belongs to the perturbed data. }
%     \label{fig:vis-pertub}
        
% \end{figure}

\begin{figure}[htbp]
    \centering
    % First row - Spectral radius
    \begin{subfigure}[b]{0.32\textwidth}
        \centering
        \includegraphics[width=\textwidth]{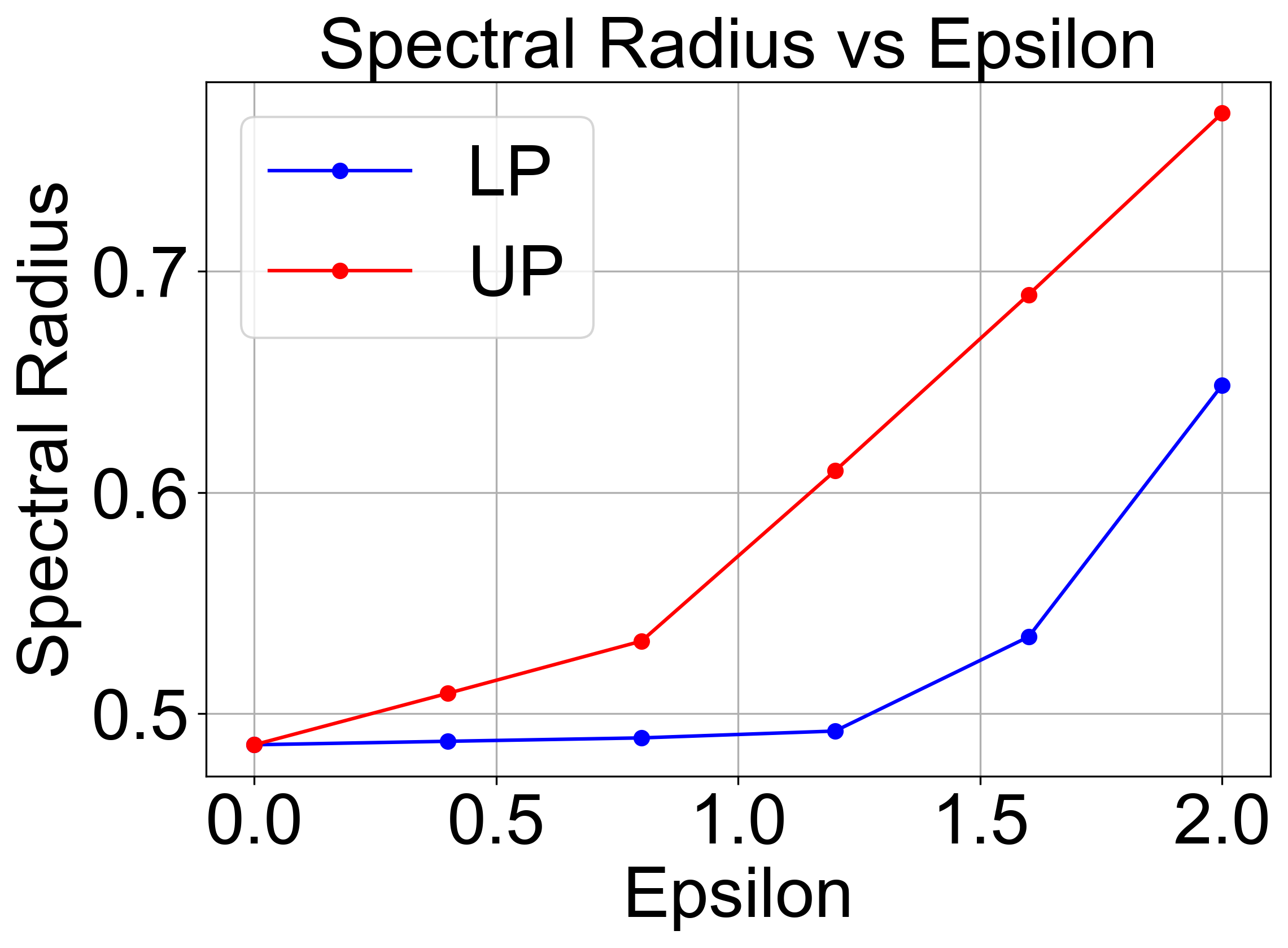}
        \caption{Spectral radius of $T_\text{Jacobi}$}
    \end{subfigure}
    \hfill
    \begin{subfigure}[b]{0.32\textwidth}
        \centering
        \includegraphics[width=\textwidth]{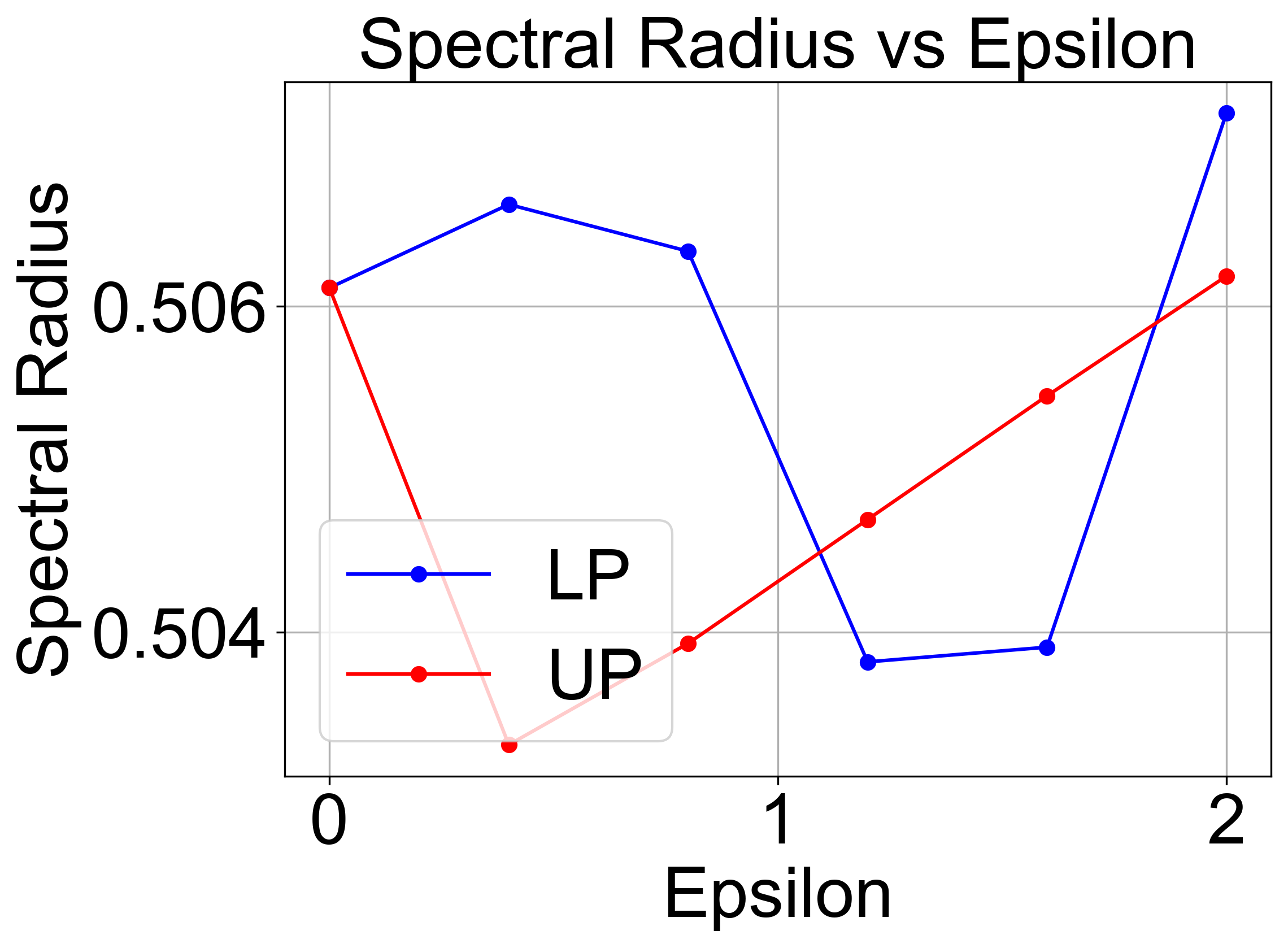}
        \caption{Spectral radius of $T_\text{GS}$}
    \end{subfigure}
    \hfill
    \begin{subfigure}[b]{0.32\textwidth}
        \centering
        \includegraphics[width=\textwidth]{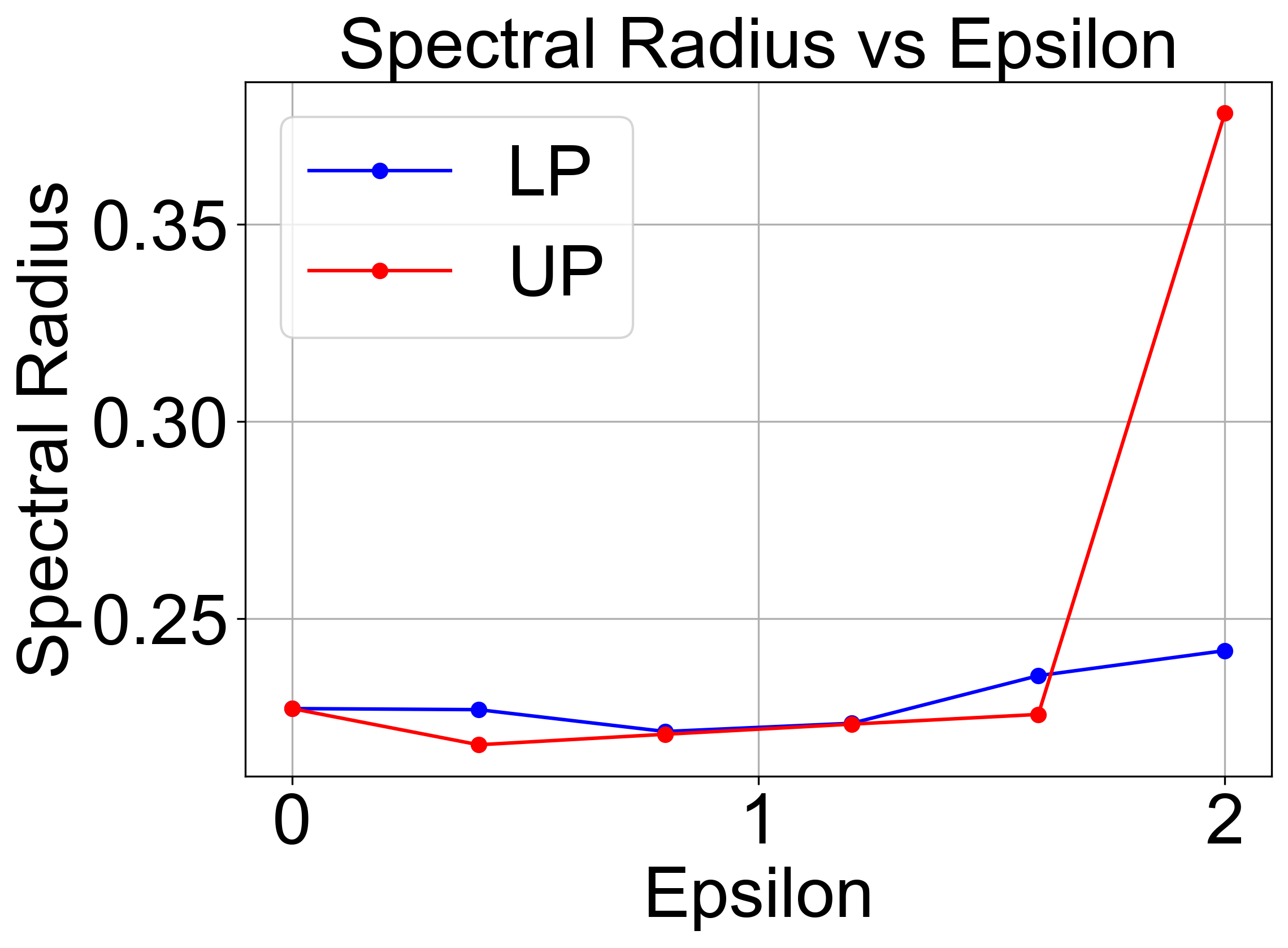}
        \caption{Spectral radius of $T_\text{SOR}$}
    \end{subfigure}
    
    \vspace{1em}
    
    % Second row - Convergence iterations
    \begin{subfigure}[b]{0.32\textwidth}
        \centering
        \includegraphics[width=\textwidth]{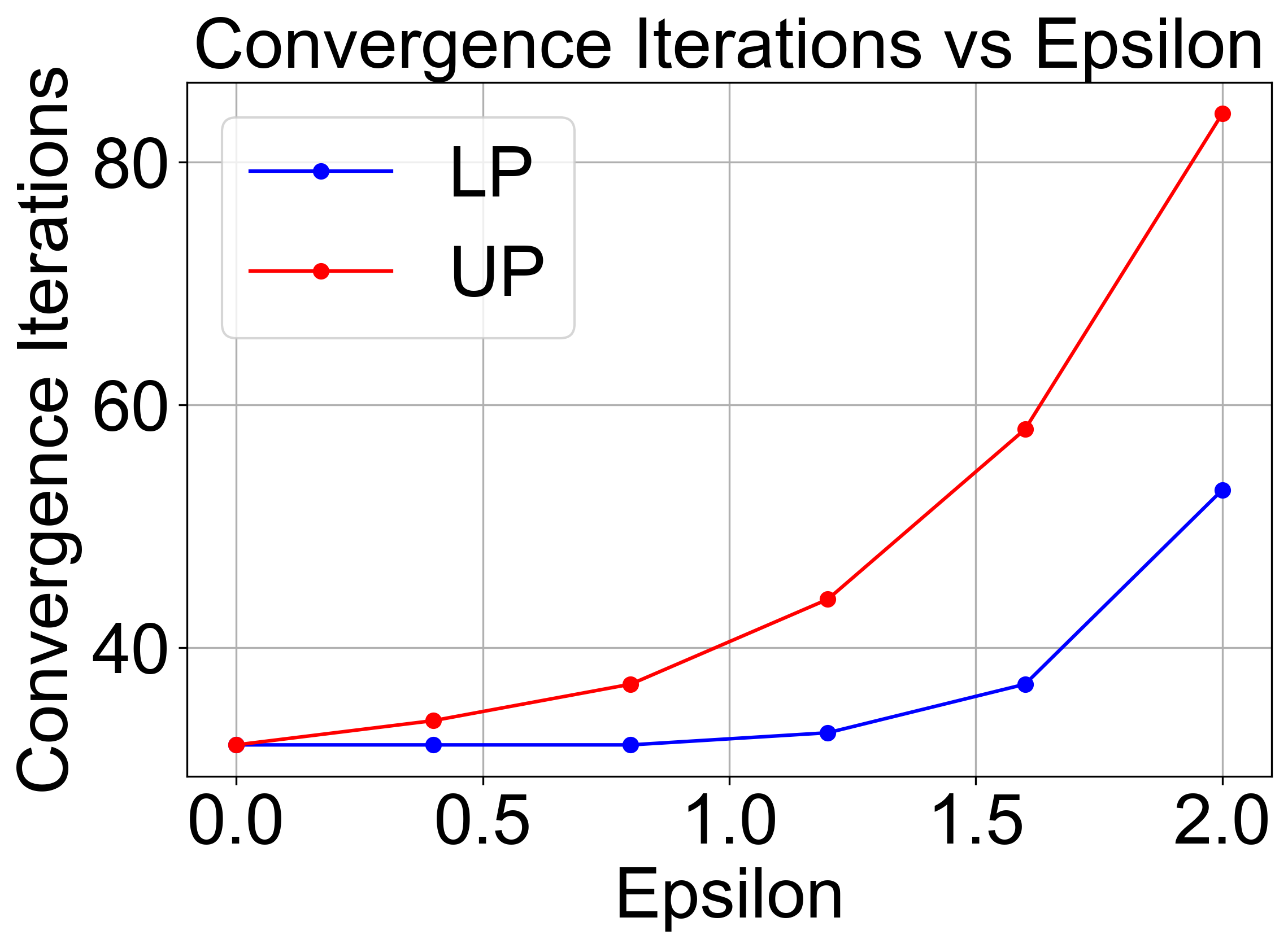}
        \caption{Iterations of Jacobi}
    \end{subfigure}
    \hfill
    \begin{subfigure}[b]{0.32\textwidth}
        \centering
        \includegraphics[width=\textwidth]{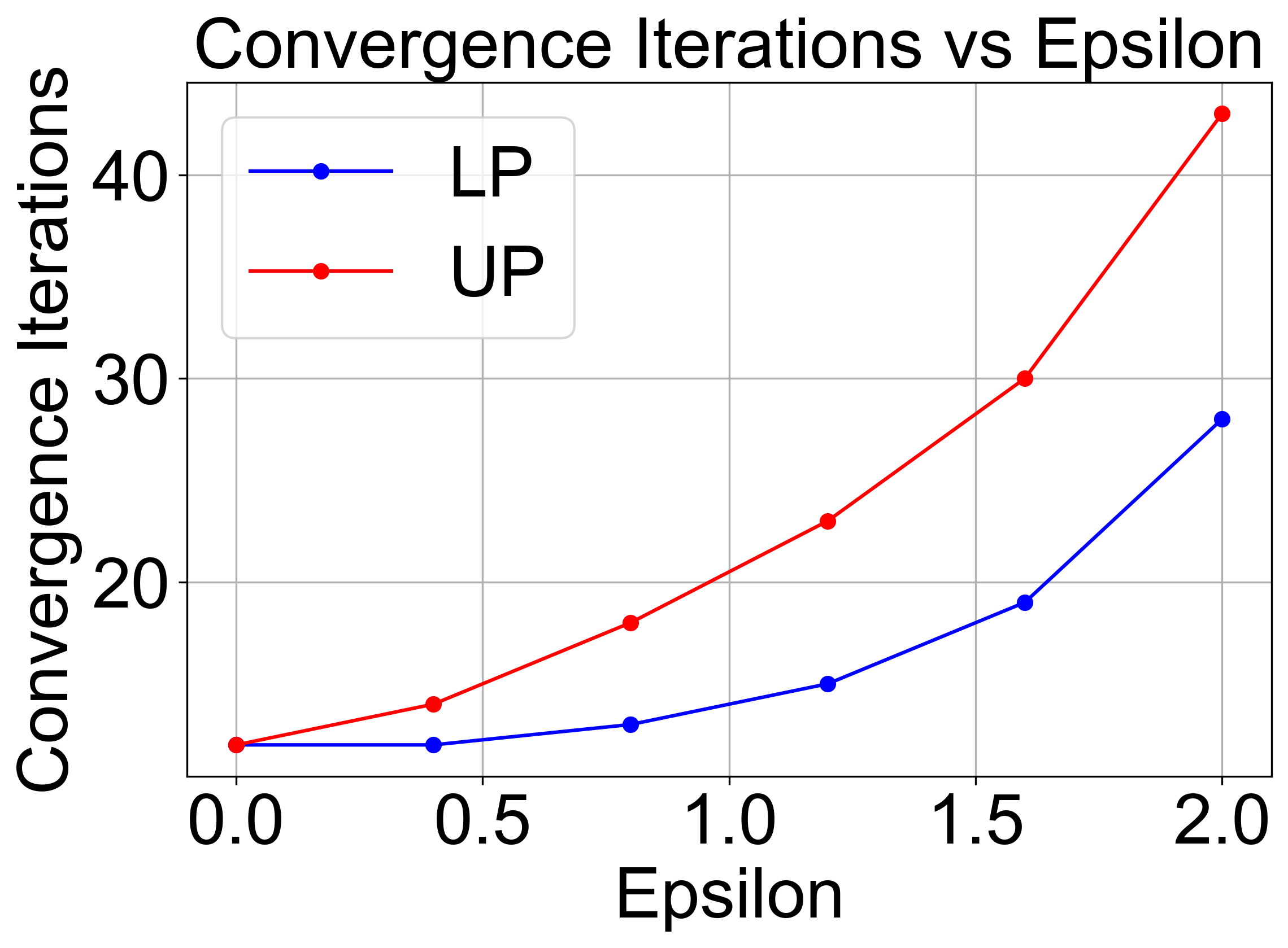}
        \caption{Gauss-Seidel iterations}
    \end{subfigure}
    \hfill
    \begin{subfigure}[b]{0.32\textwidth}
        \centering
        \includegraphics[width=\textwidth]{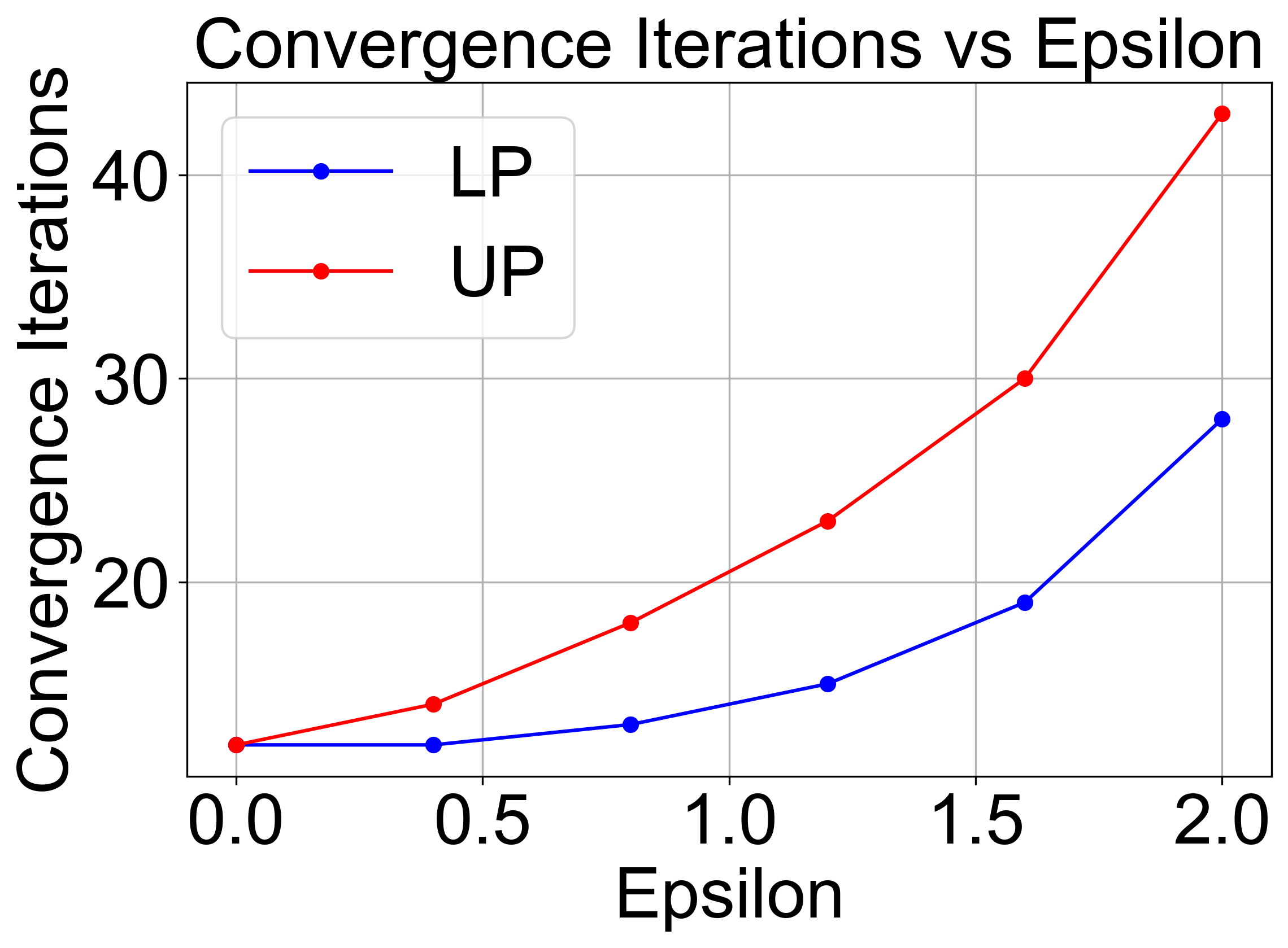}
        \caption{SOR iterations}
    \end{subfigure}
    \caption{Spectral radius and convergence iterations for Jacobi, Gauss-Seidel, and SOR methods under different perturbations.}
    \label{fig:spectral_radius_change}
\end{figure}

\begin{figure}[thbp]
    \centering
    \begin{subfigure}[b]{0.24\textwidth}
        \centering
        \includegraphics[width=\textwidth]{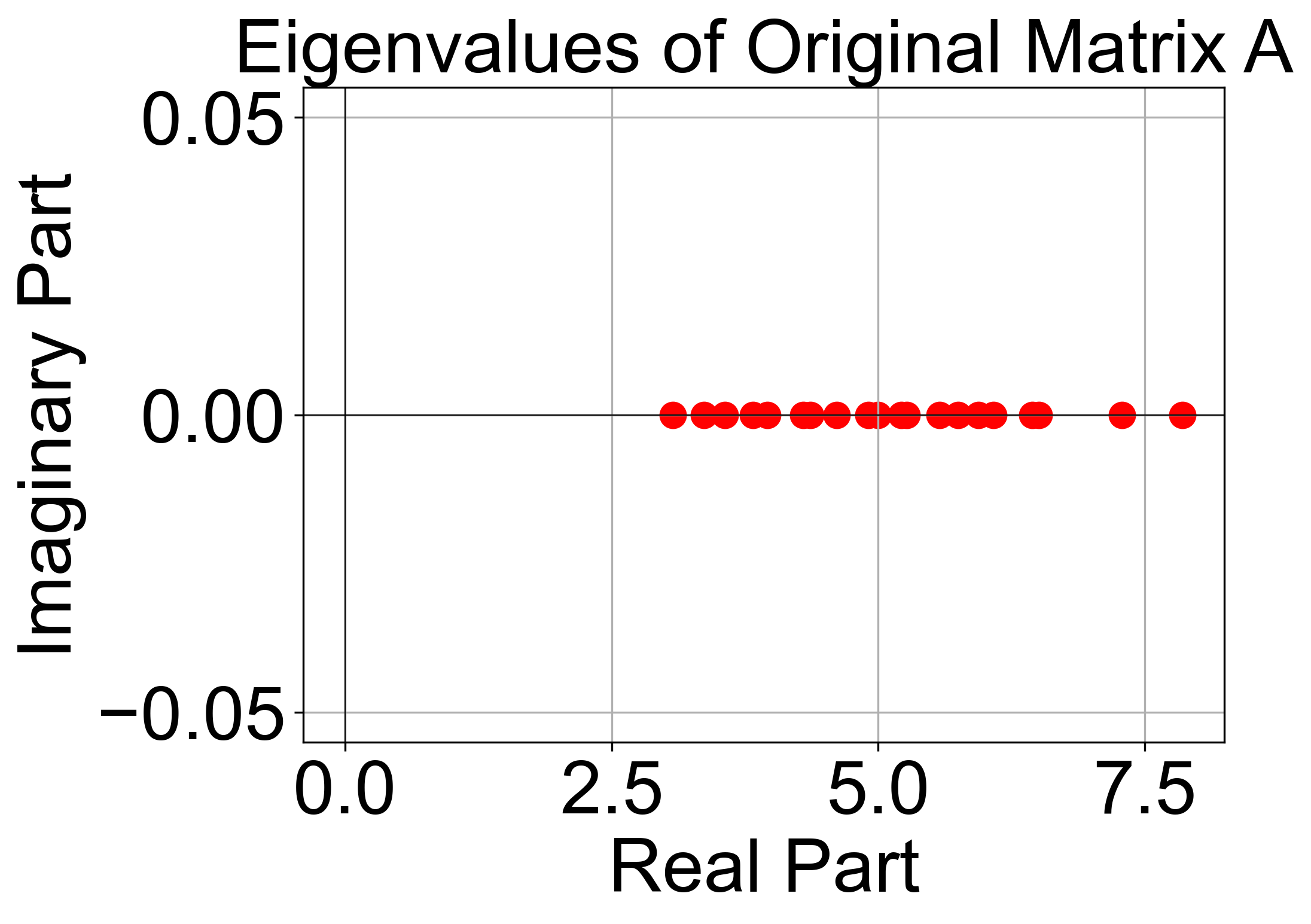}
        \caption{Original eigenvalues}
    \end{subfigure}
    \hfill
    \begin{subfigure}[b]{0.24\textwidth}
        \centering
        \includegraphics[width=\textwidth]{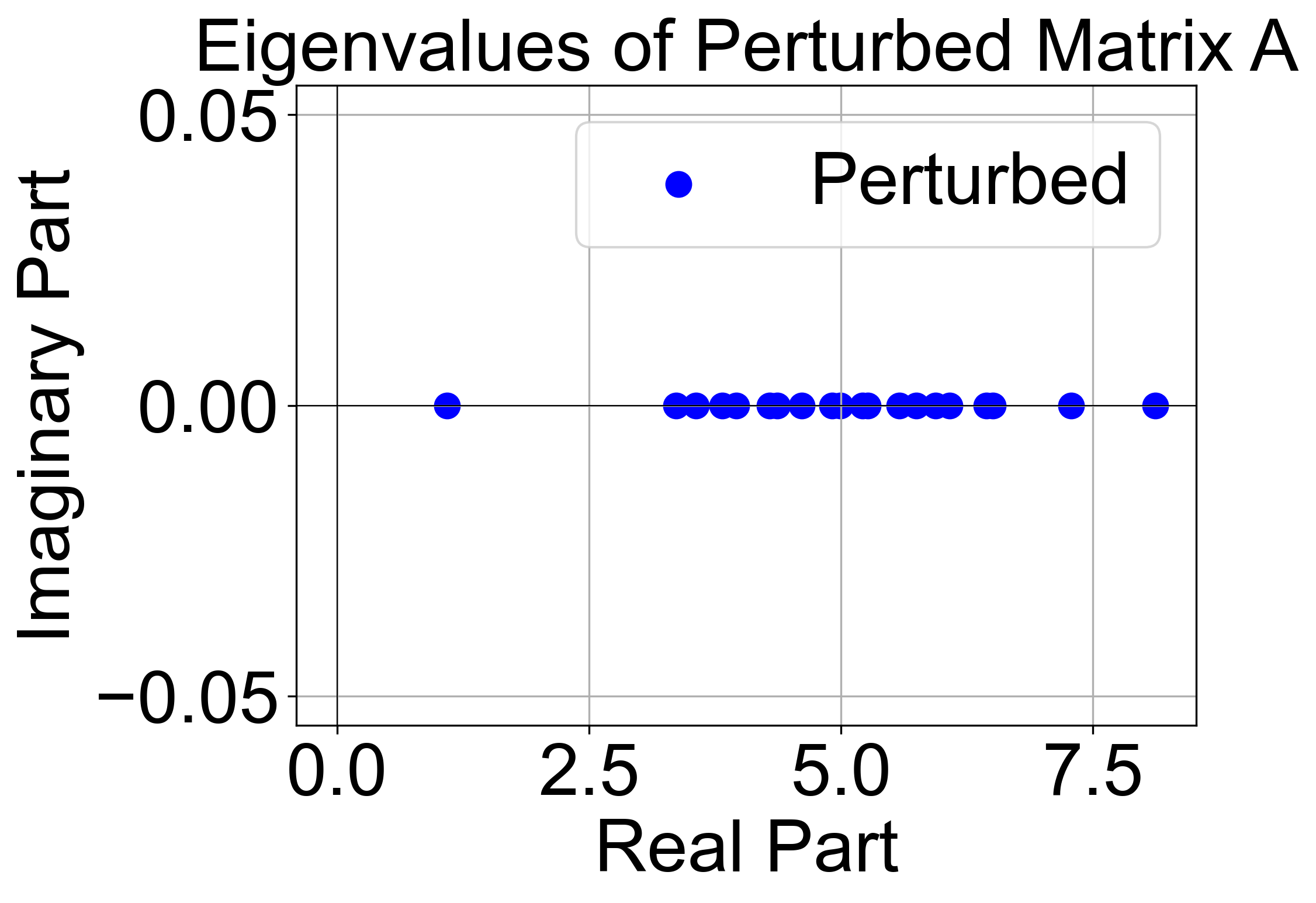}
        \caption{Perturbed eigenvalues}
    \end{subfigure}
    \hfill
    \begin{subfigure}[b]{0.24\textwidth}
        \centering
        \includegraphics[width=\textwidth]{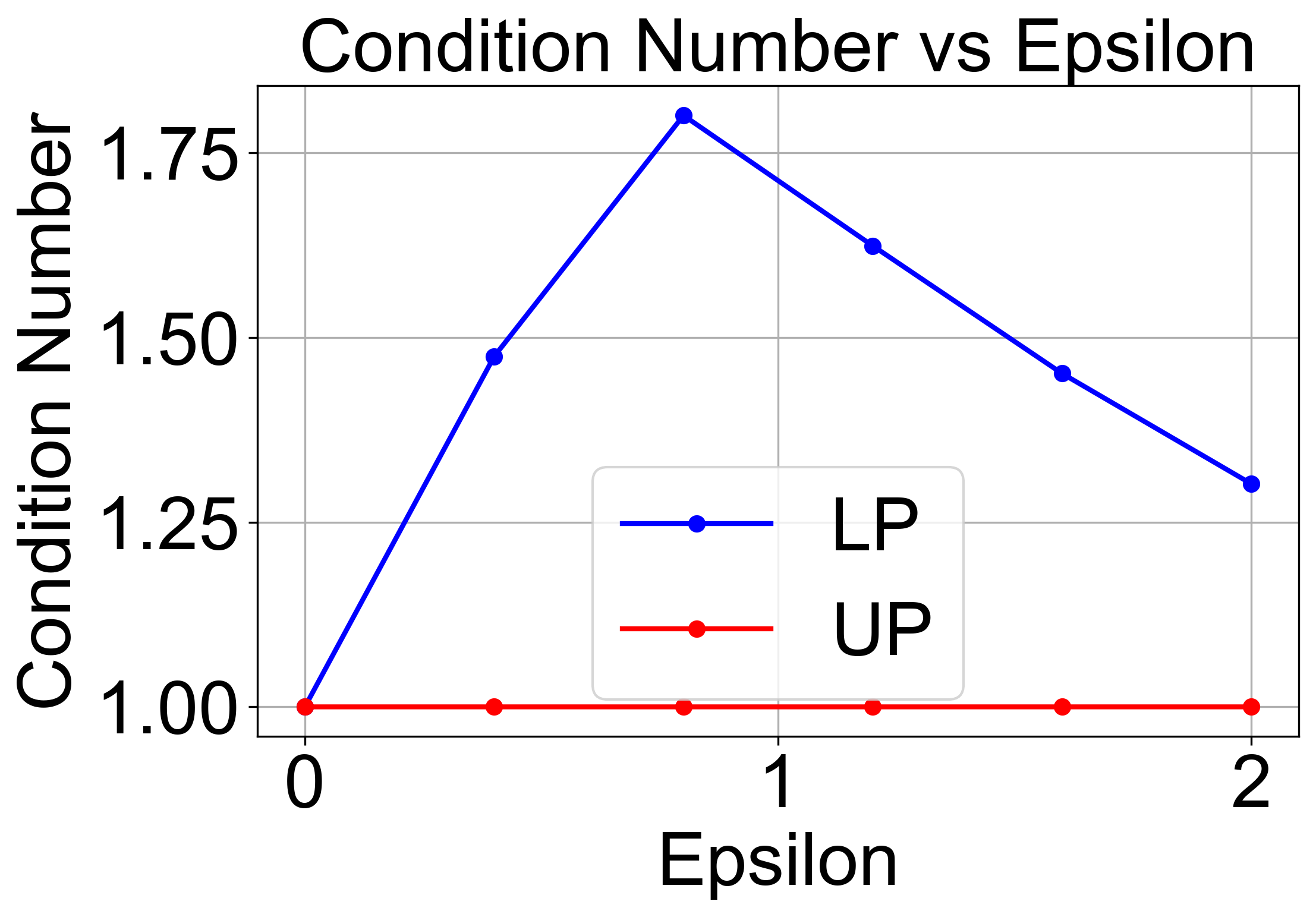}
        \caption{Condition number}
    \end{subfigure}
    \hfill
    \begin{subfigure}[b]{0.24\textwidth}
        \centering
        \includegraphics[width=\textwidth]{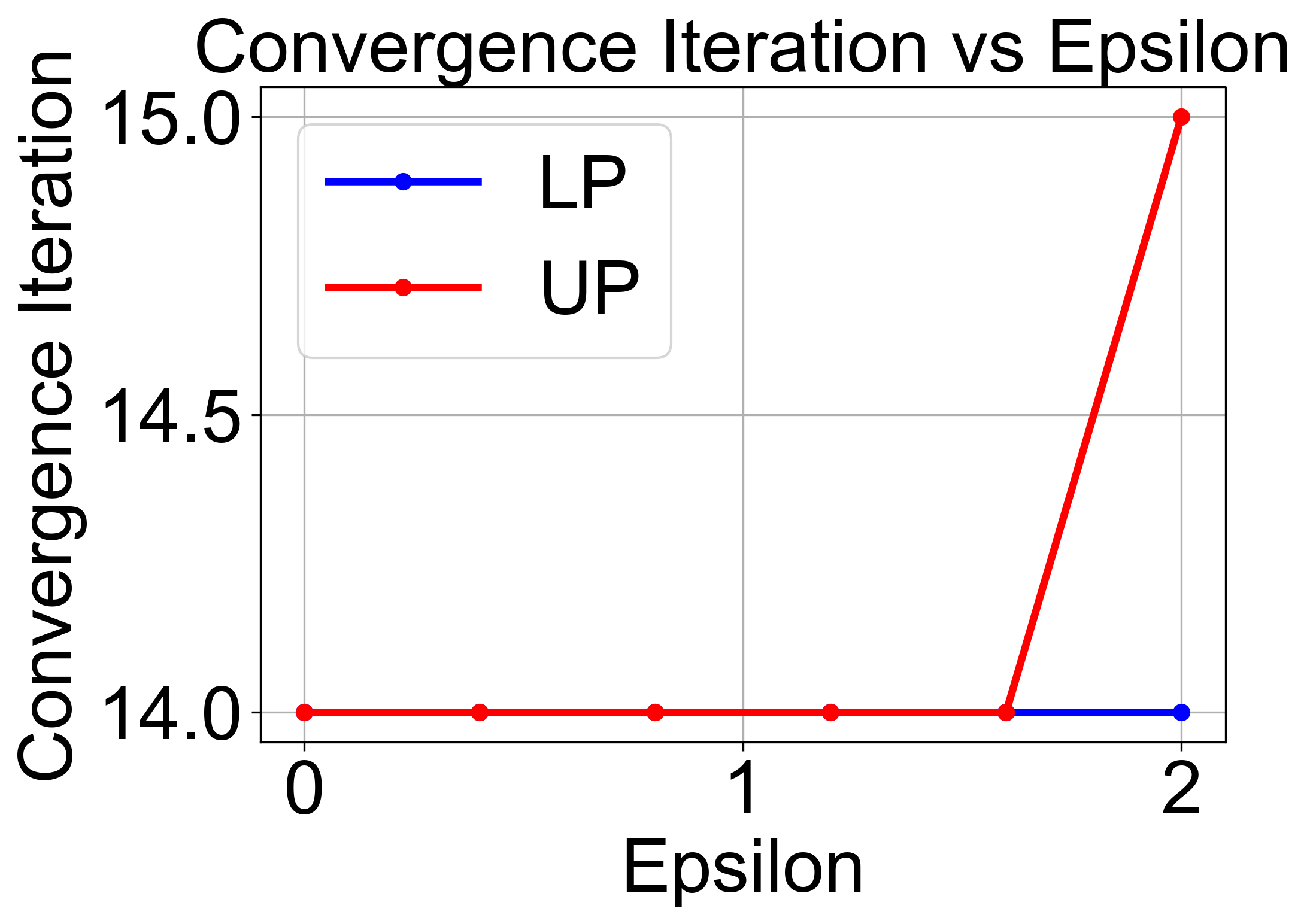}
        \caption{Convergence iterations}
    \end{subfigure}
    \caption{GMRES analysis: eigenvalue distribution, the condition number of eigenvectors of a perturbed matrix, and convergence iterations under perturbations.}
    \label{fig:eigenvalue_distribution}
\end{figure}
\subsection{Results and Analysis on Iterative Solver}
\label{sec.iterative_solver}
\header{Perturbation Decreases the System Fitting Accuracy}
The results are presented in Table \ref{tab:main_iterative_results}. In general, we observe that the system output error and residual increase as the perturbation radius $\epsilon$ grows. Unlike the results with direct solvers, we find that LP consistently causes more significant performance degradation than UP across the iterative solvers. \textit{We attribute this to the iterative solvers' reliance on the residual vector to refine the solution at each iteration, making them more susceptible to perturbations that directly affect the right-hand side vector $b$ (or output $y_t$ in our context).} While UP can also introduce errors on the right-hand side, it does not fully perturb it as LP does. 

\header{Perturbation Slowing Down Convergence} Regarding convergence behavior, we visualize the convergence of the five iterative solvers under both LP and UP with varying $\epsilon$ values in Figure \ref{fig:convergence_behavior}. Across different trials with various noise radii, we consistently observe that LP and UP both lead to slower convergence for most of the solvers as $\epsilon$ increases. 
However, we find that GMRES is the only solver that remains stable in terms of convergence speed upon perturbation, though the converged solution's accuracy is affected. We hypothesize that this robustness is attributed to GMRES’s residual minimization property and its use of orthogonal Krylov subspaces, which make it less sensitive to perturbations and changes in the system’s condition compared with other iterative methods under similar perturbations. Among all the iterative solvers, we find that GD is the most sensitive to perturbation.

\header{Analysis of Convergence Slowdown} We now provide analysis and visualization to understand the root cause of convergence slowdown. For the three basic iterative solvers like Jacobi, Gauss-Seidel and SOR, the convergence rate is determined by the spectral radius $\rho(T)$ of the iteration matrix $T$ ($T_\text{Jacobi}=D^{-1}R, T_\text{GS}=-(D+L)^{-1}U, T_\text{SOR}=(D+\omega L)^{-1}((1-\omega)D - \omega U)$ ). The condition of $\rho(T) < 1$ is the necessary and sufficient condition for convergence, and when the spectral radius is closer to 1, the convergence speed will slow down. We visualize the change in spectral radius for these three solvers under different perturbations in Figure \ref{fig:spectral_radius_change}. We observe that the spectral radius of the three solvers increases in general as the perturbation radius $\epsilon$ increases.

\begin{figure}[thbp]
    \centering
    \begin{subfigure}[b]{0.22\textwidth}
        \centering
        \includegraphics[width=\textwidth]{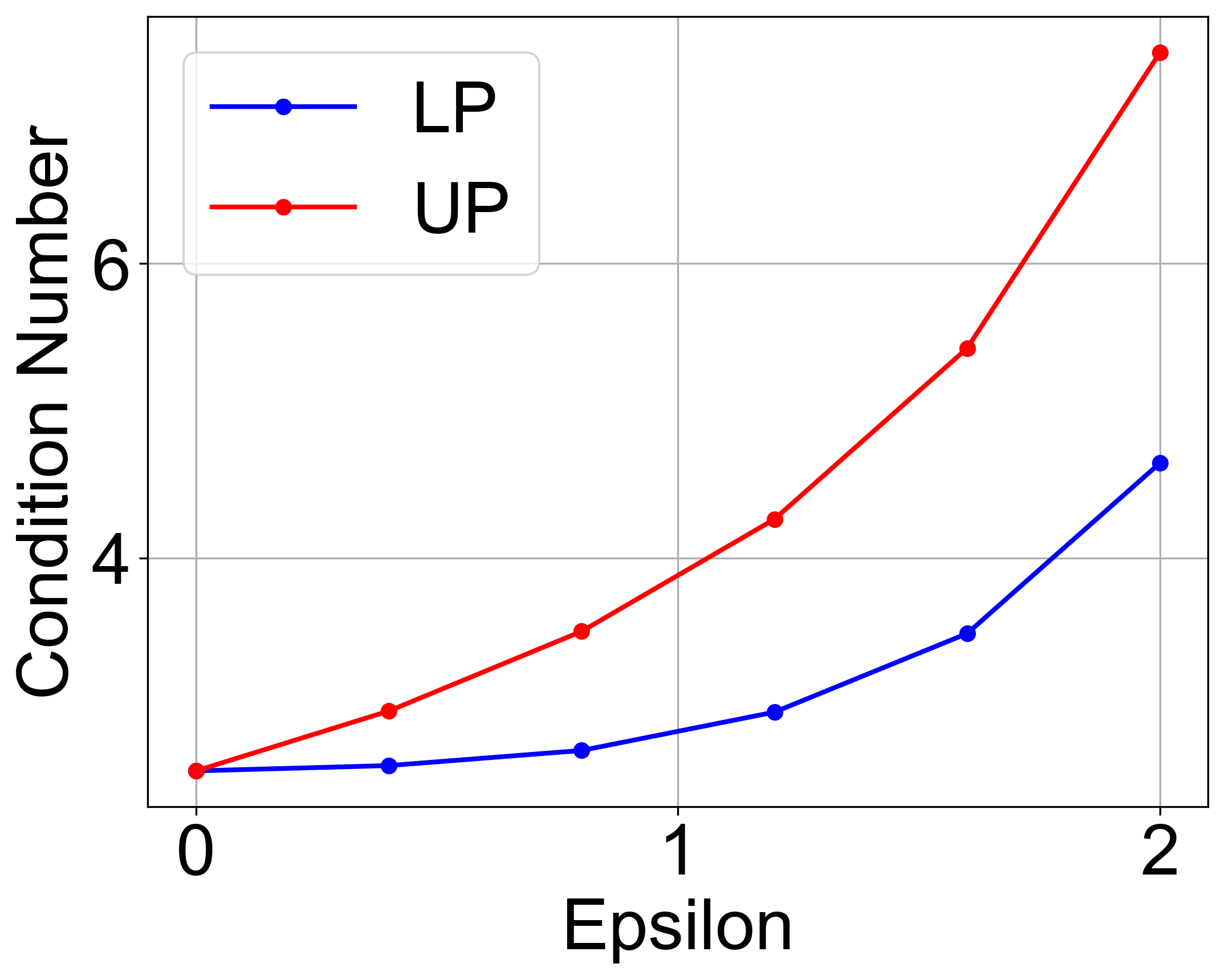}
        \caption{Condition number}
    \end{subfigure}
    \hfill
    \begin{subfigure}[b]{0.26\textwidth}
        \centering
        \includegraphics[width=\textwidth]{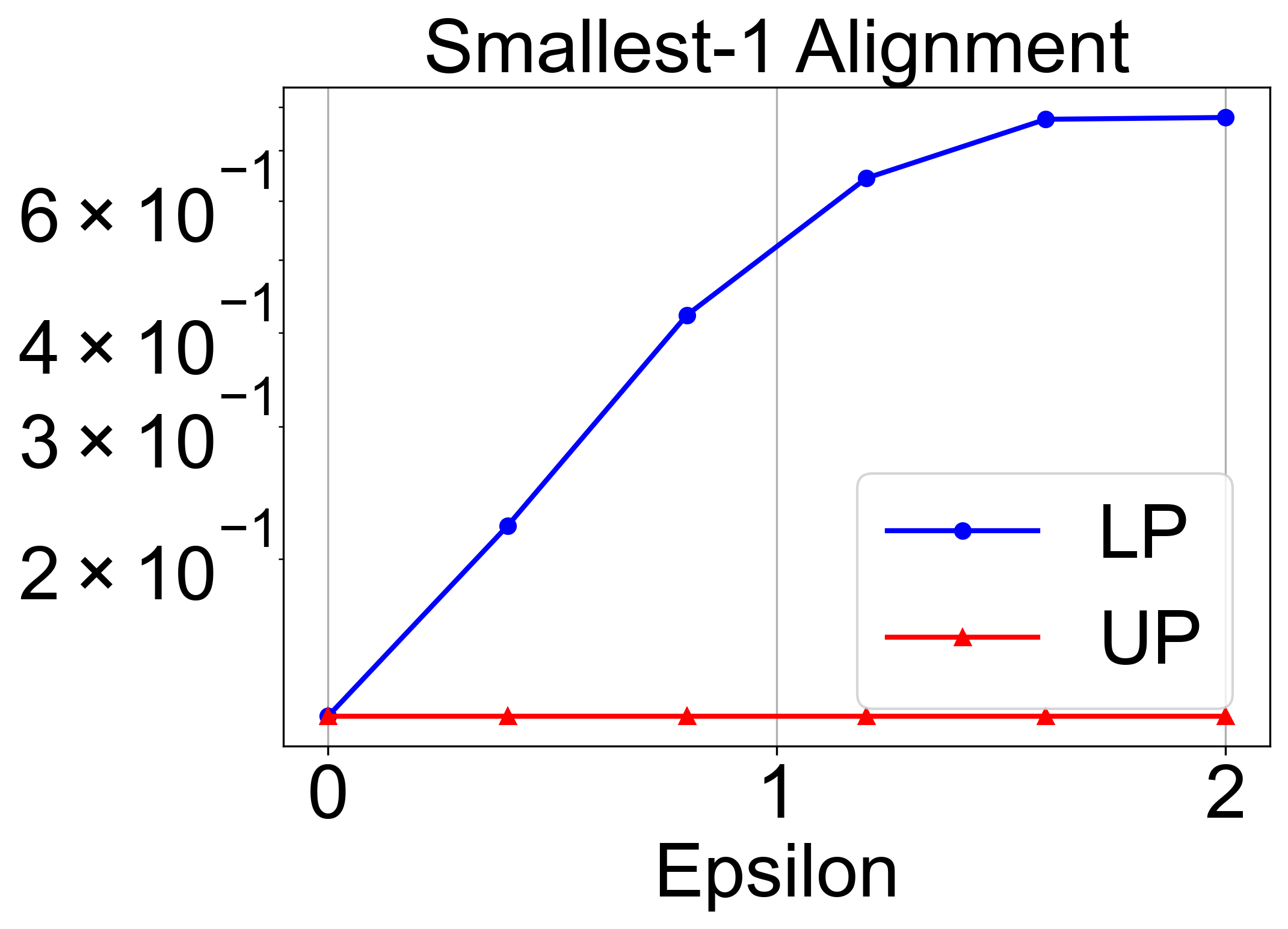}
        \caption{Smallest-1 alignment}
    \end{subfigure}
    \hfill
    \begin{subfigure}[b]{0.26\textwidth}
        \centering
        \includegraphics[width=\textwidth]{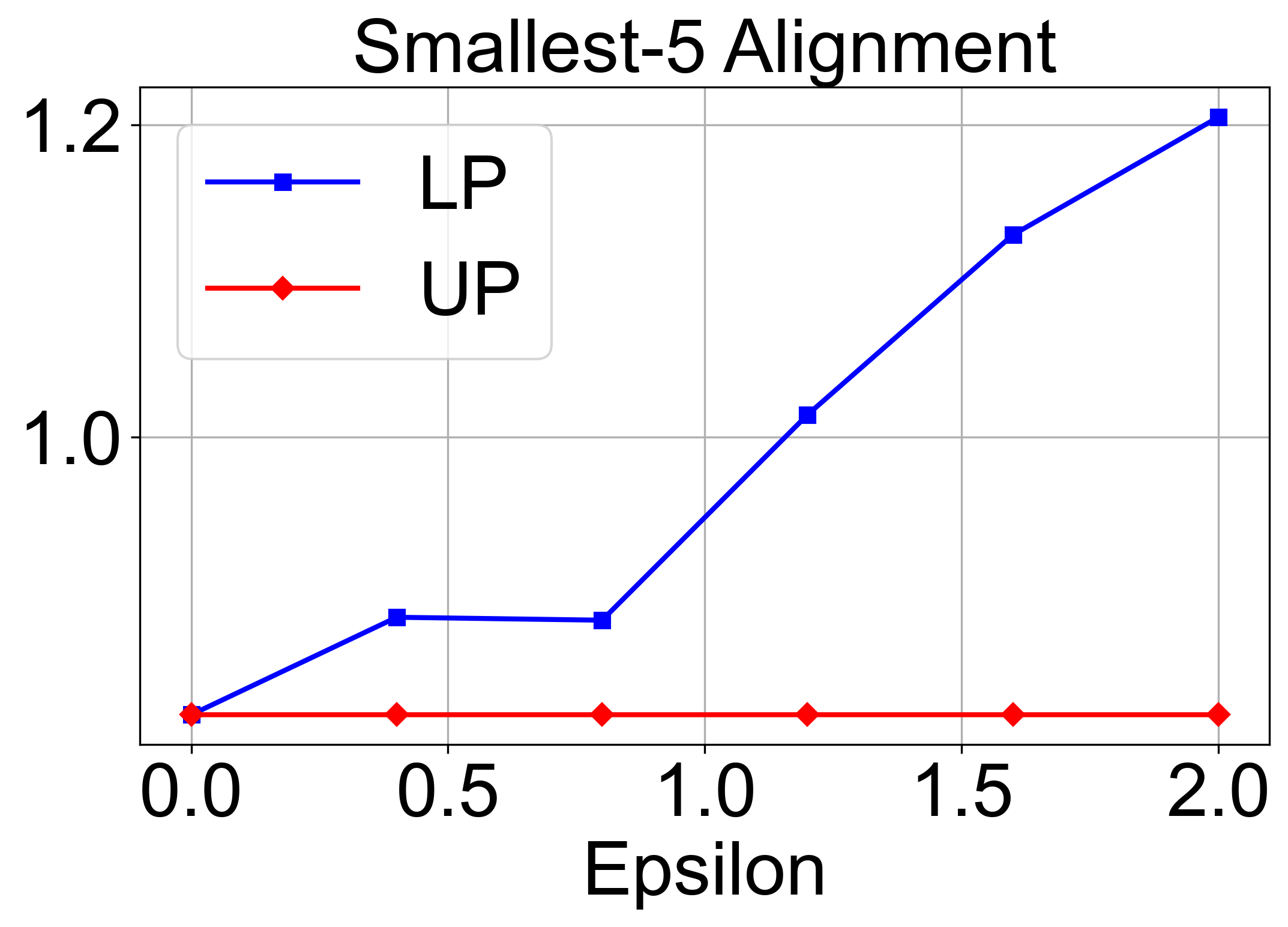}
        \caption{Smallest-5 alignment}
    \end{subfigure}
    \hfill
    \begin{subfigure}[b]{0.22\textwidth}
        \centering
        \includegraphics[width=\textwidth]{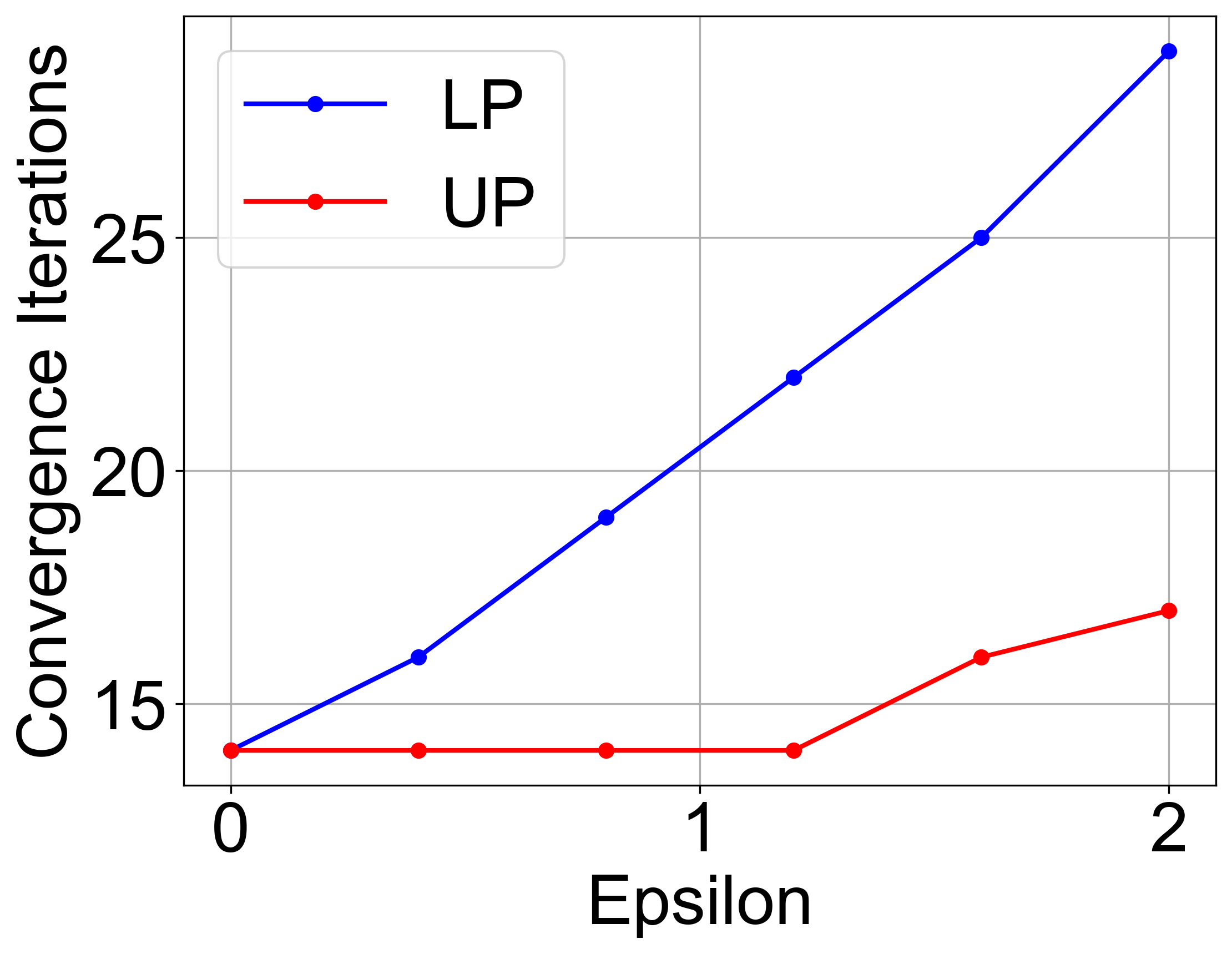}
        \caption{Convergence iterations}
    \end{subfigure}
    \caption{CG analysis: condition number, eigenvalue alignments, and convergence iterations under perturbations.}
    \label{fig:cg_alignment}
\end{figure}

For the Krylov subspace methods, GMRES, we visualize the eigenvalue distribution of the coefficient matrix (the original and those most perturbed by UP or LP) and the condition number of eigenvectors in Figure \ref{fig:eigenvalue_distribution}. As we can see, even though LP makes the eigenvectors slightly more ill-conditioned, the clustering properties of the eigenvectors are still relatively similar. As for UP, we find that it does not introduce an effect in terms of the condition number of eigenvectors. These may explain why GMRES is robust to both LP and UP; see \citep{saad1986gmres,embree2022descriptive} for more discussion on GMRES convergence theory.

For the energy-minimizing method, CG, the convergence is dominated by the condition number of the system matrix $A$ and the relative alignment of the initial residual with the eigenvectors of $A$ that correspond to the smallest eigenvalues. We analyze the alignment of the initial residual with the eigenvectors of $A$. We compute this by projecting the initial residual $r_0$ onto each eigenvector $v_i$ of $A$, given by $c_i = (r_0^T v_i) v_i$. To quantify the alignment, we calculate the squared magnitude of these projections: i) Alignment with the smallest eigenvalue: $\text{Alignment with smallest eigenvalue} = \|c_1\|^2$, where $v_1$ corresponds to the smallest eigenvalue $\lambda_1$. ii) Alignment with the five smallest eigenvalues: $\text{Alignment with five smallest eigenvalues} = \sum_{i=1}^5 \|c_i\|^2$, where $v_1, \ldots, v_5$ correspond to the five smallest eigenvalues. As shown in Figure \ref{fig:cg_alignment}, the LP perturbation greatly increases this alignment in both cases, leading to slower convergence. In contrast, the UP does not affect the alignment much. This reaffirms the results in Table \ref{tab:main_iterative_results} that LP causes convergence slowdown more than UP.

\begin{figure}[thbp]
    \centering
    \begin{subfigure}[b]{0.28\textwidth}
        \centering
        \includegraphics[width=\textwidth]{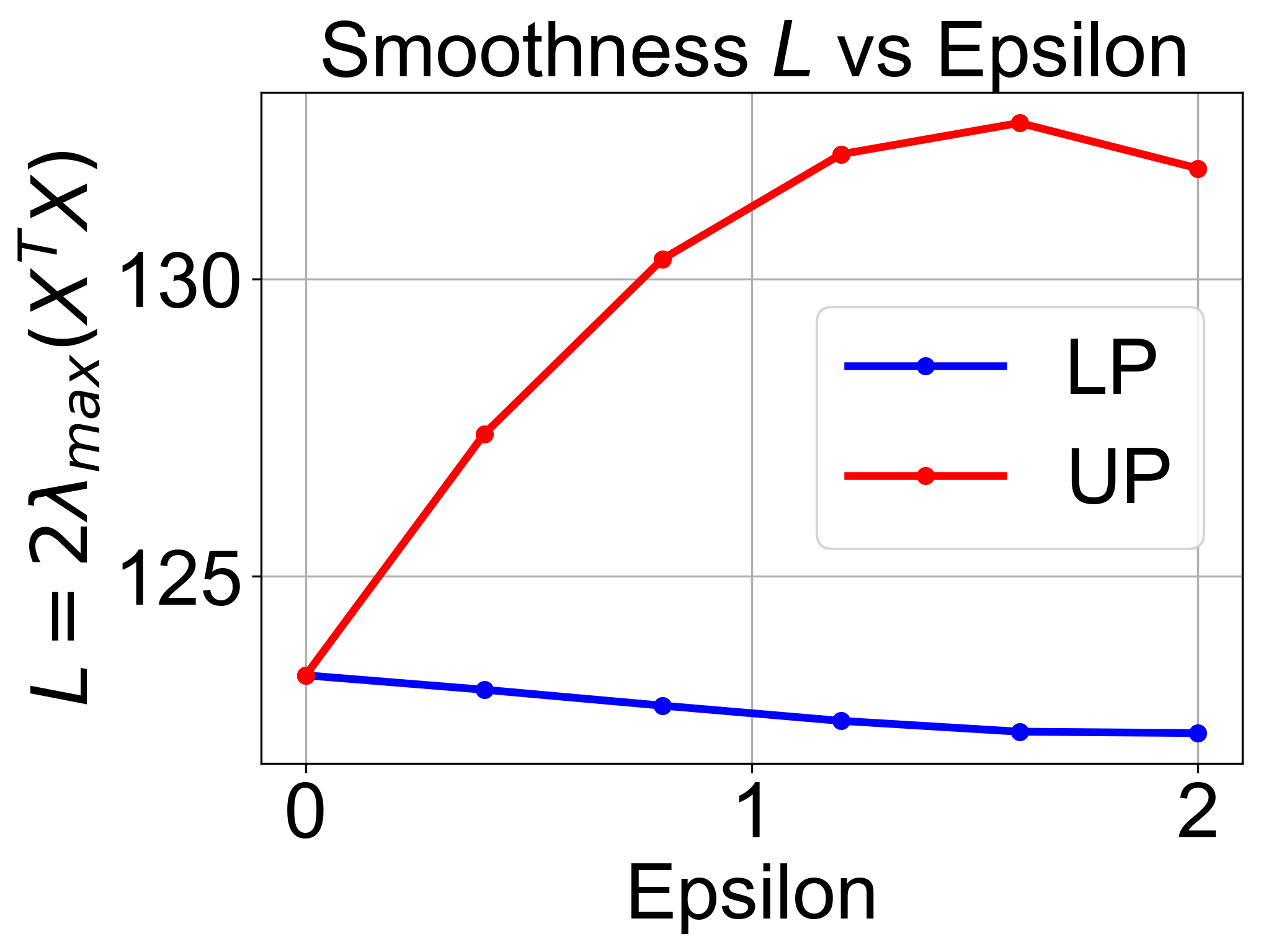}
        \caption{Smoothness constant}
        \label{fig:gd_smoothness}
    \end{subfigure}
    \hfill
    \begin{subfigure}[b]{0.68\textwidth}
        \centering
        \includegraphics[width=\textwidth]{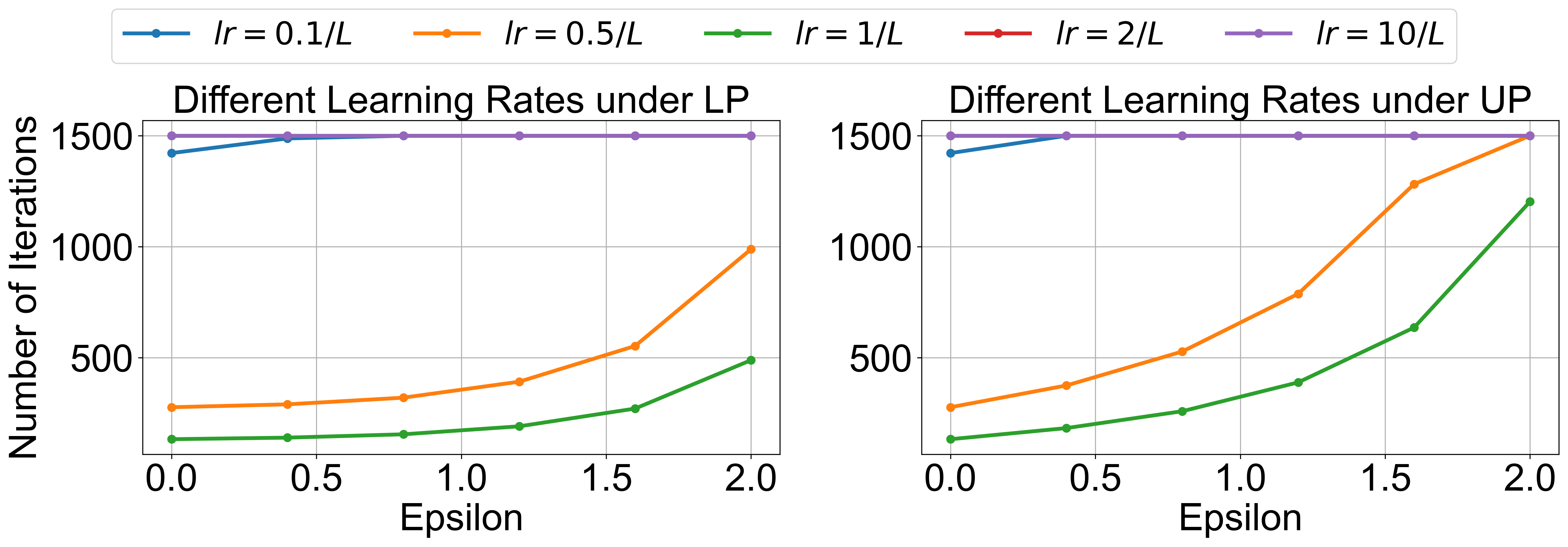}
        \caption{Convergence iterations with different learning rates}
        \label{fig:gd_lr}
    \end{subfigure}
    \caption{Gradient descent analysis: the smoothness constant $L$ and convergence iterations under perturbations with various learning rates.}
    \label{fig:gd_analysis}
\end{figure}
For gradient descent, the convergence rate is affected by the Lipschitz constant of the underlying function $f$. Since the gradient is $\nabla f(x_k) = 2X^T(Xx_k-b)$, the Lipschitz constant of the underlying function is $L = 2\lambda_{max}(X^TX)$. We visualize this Lipschitz constant and the convergence iterations with different learning rates $lr$ in Figure \ref{fig:gd_analysis}. As we can see, UP introduces a larger smoothness factor increase, leading to a more significant convergence slowdown than LP. Furthermore, from the iteration number, we can see that GD is sensitive to the learning rate setup, and setting it to $1/L$ is an empirically good choice that achieves the fastest convergence.

\begin{wrapfigure}{r}{0.5\textwidth}
  \centering
  \includegraphics[width=\linewidth]{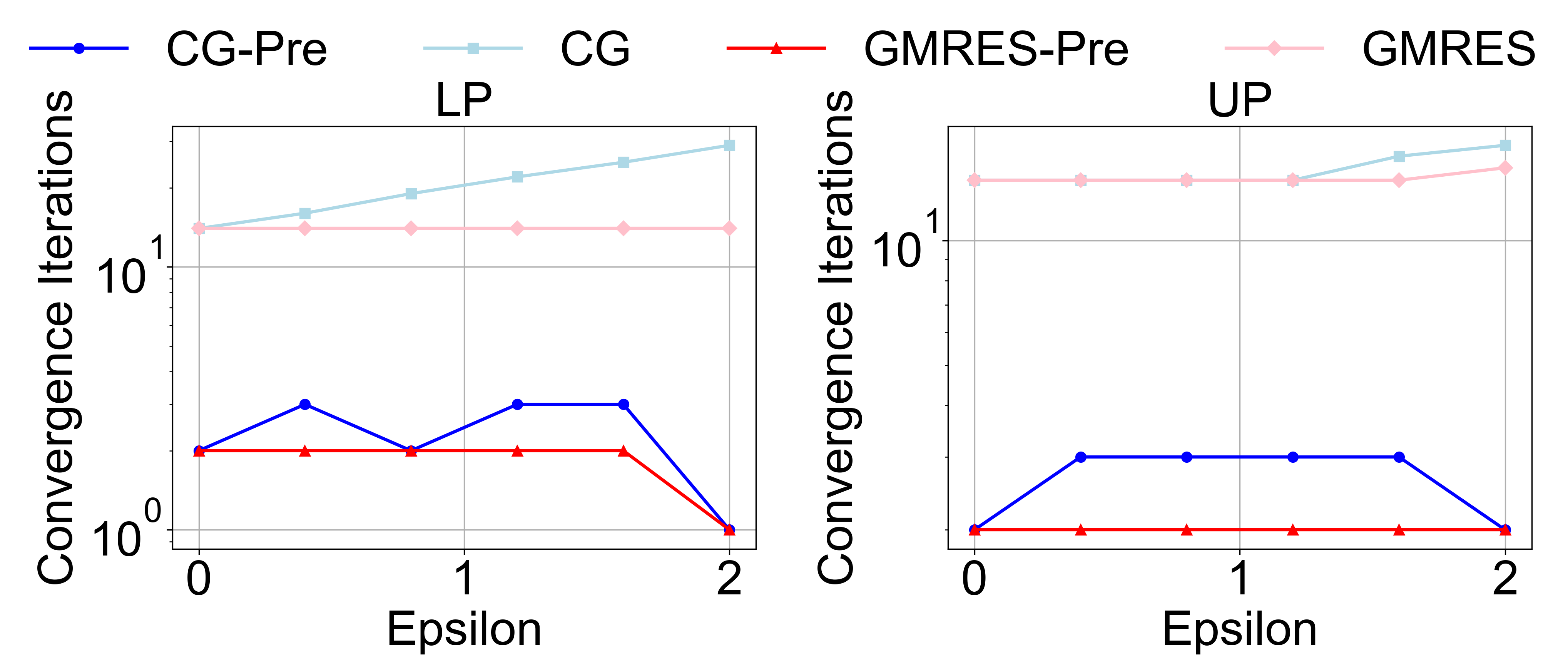}
  \caption{Convergence behavior with and without ILU preconditioning under various perturbation levels. Label with `Pre' means the preconditioning is applied.}
  \label{fig:preconditioner}
\end{wrapfigure}
\header{Using Preconditioning as Defense for Stabilizing Convergence} Preconditioning improves the convergence of iterative solvers, especially under perturbations; see \citet{benzi2002preconditioning} for the survey of preconditioning techniques as well as \citet{CardeSGug2021,CarDiss2021} and references therein on preconditioner update strategies. We focused on the Incomplete LU (ILU) factorization preconditioner, which approximates the original matrix A as $\tilde{A} = \tilde{L}\tilde{U}$, resulting in a preconditioned system $M^{-1}Ax = M^{-1}b$, where $M = \tilde{L}\tilde{U}$. Our results show that applying this preconditioner can significantly stabilize the convergence of iterative solvers. 
Figure \ref{fig:preconditioner} illustrates the ILU preconditioner's impact on convergence behavior. It substantially improves convergence rates, particularly with higher perturbation levels, by reducing the system matrix's condition number and mitigating perturbation effects on the eigenvalue distribution. In conclusion, appropriate preconditioning, such as ILU, serves as an effective defense against perturbations in linear systems, enhancing the stability of iterative solvers in practical applications with uncertainties.

\begin{figure}[thbp]
    \centering
    \begin{subfigure}[b]{0.32\textwidth}
        \centering
        \includegraphics[width=\textwidth]{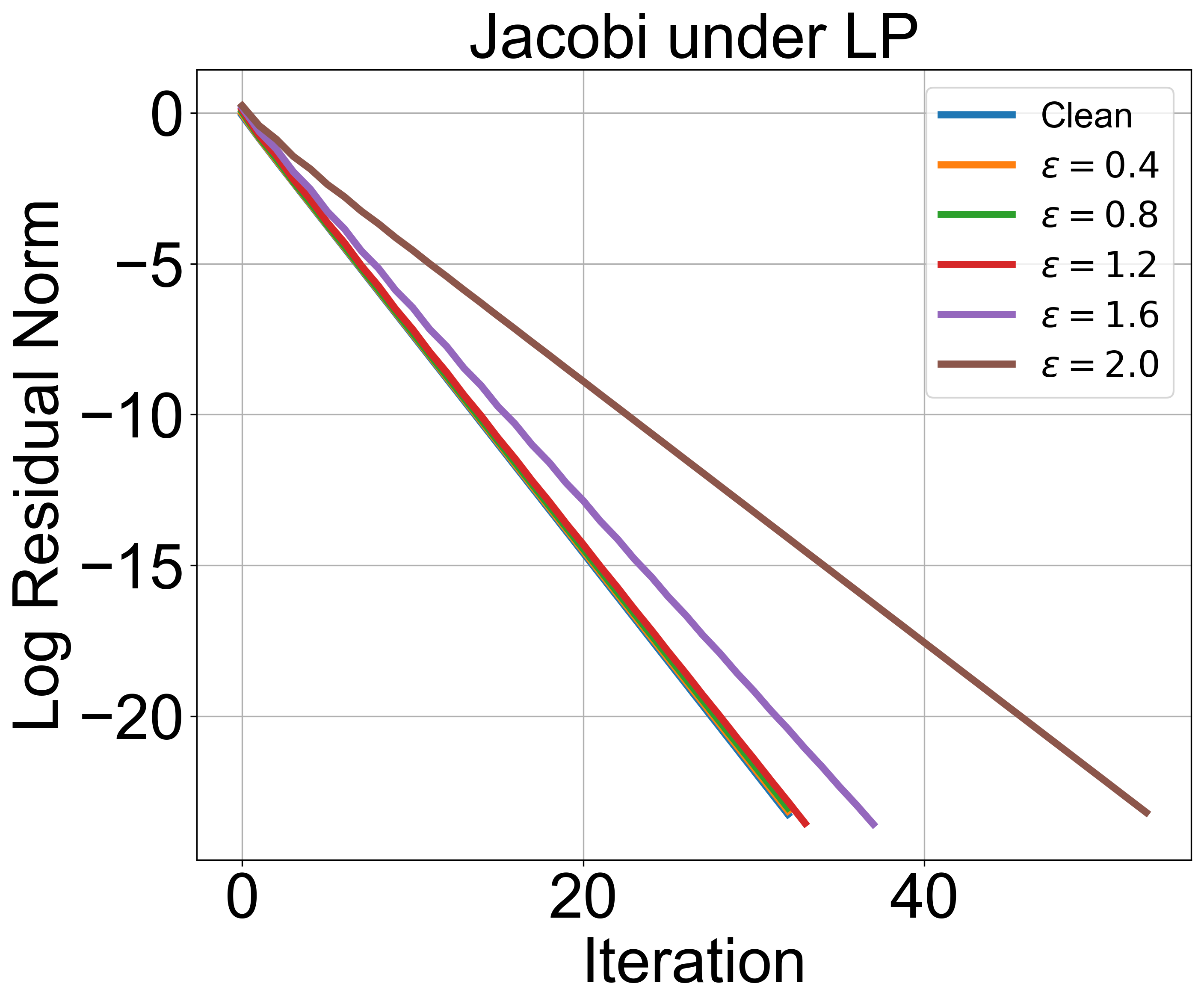}
        \caption{Jacobi (LP)}
        \label{fig:jacobi_LP}
    \end{subfigure}
    \hfill
    \begin{subfigure}[b]{0.32\textwidth}
        \centering
        \includegraphics[width=\textwidth]{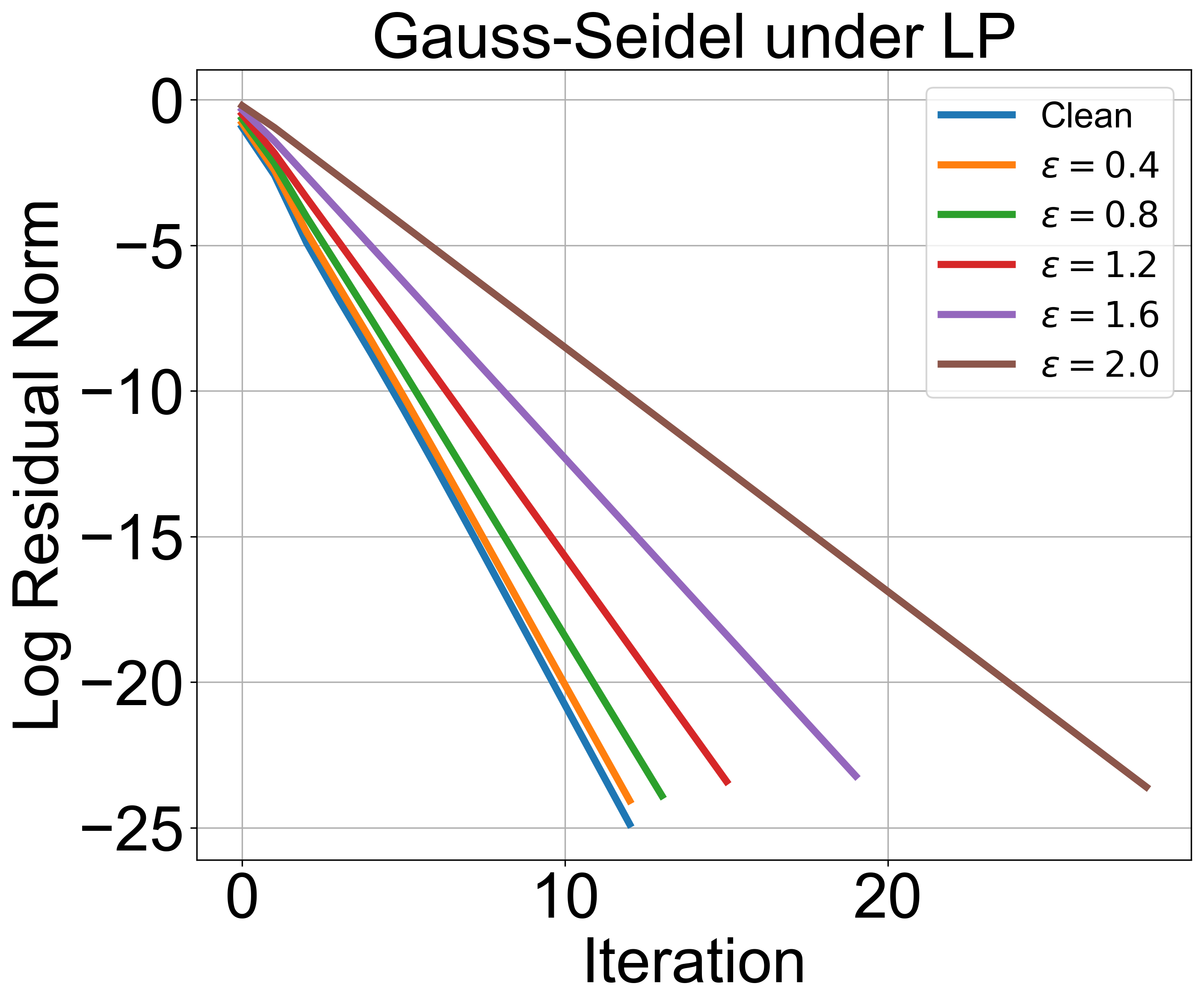}
        \caption{Gauss-Seidel (LP)}
        \label{fig:gauss_LP}
    \end{subfigure}
    \hfill
    \begin{subfigure}[b]{0.32\textwidth}
        \centering
        \includegraphics[width=\textwidth]{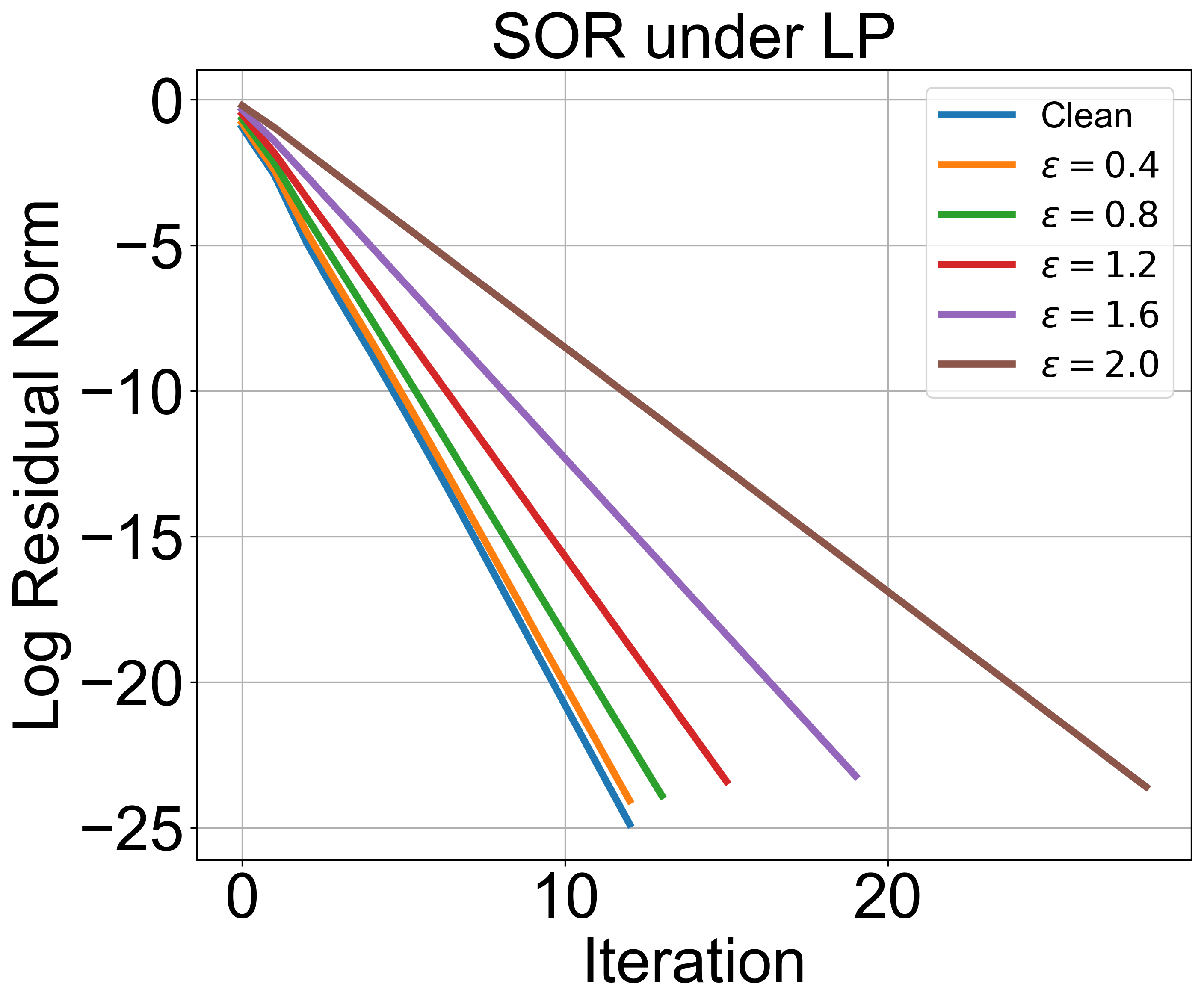}
        \caption{SOR (LP)}
        \label{fig:sor_LP}
    \end{subfigure}

    \begin{subfigure}[b]{0.32\textwidth}
        \centering
        \includegraphics[width=\textwidth]{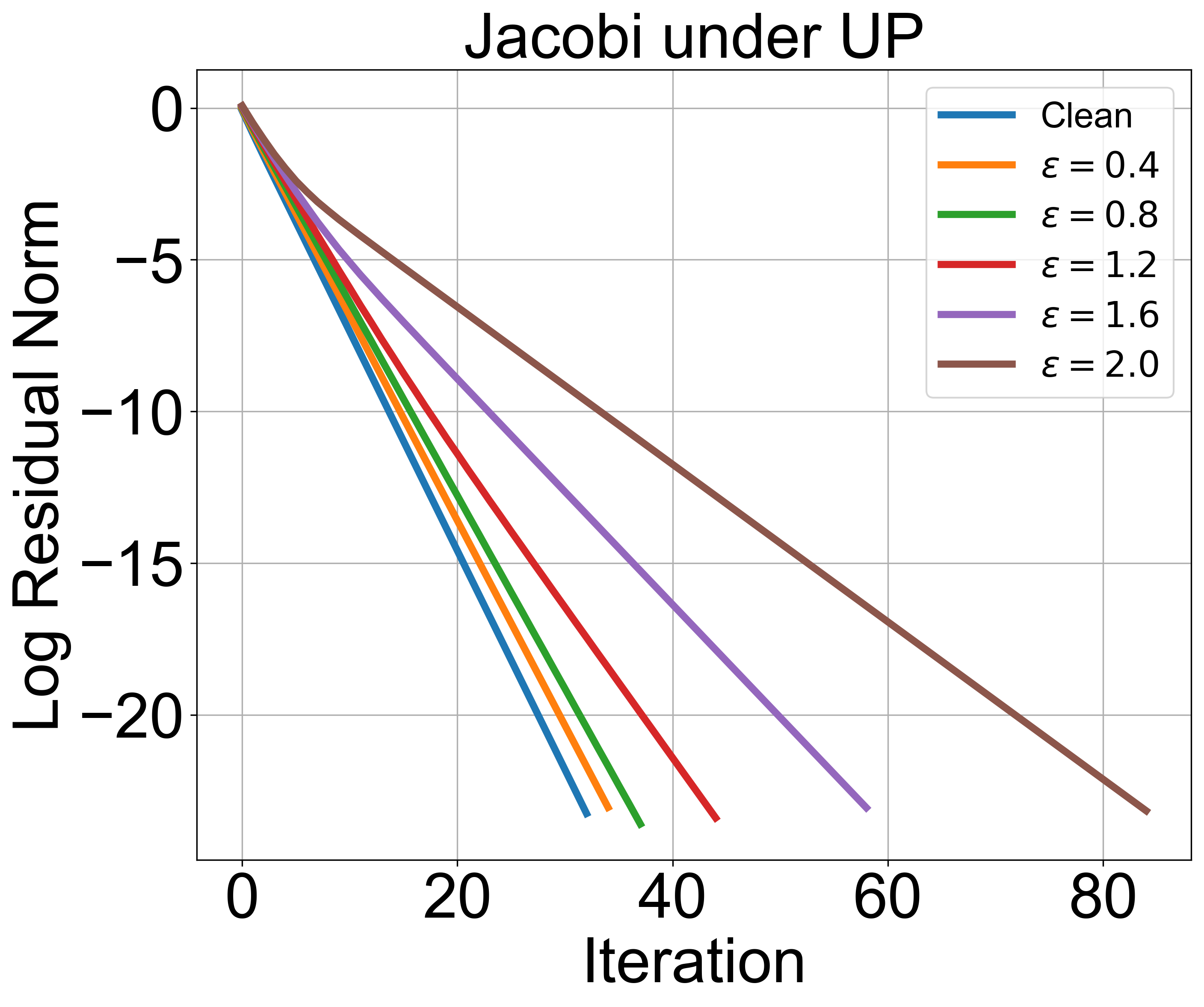}
        \caption{Jacobi (UP)}
        \label{fig:jacobi_UP}
    \end{subfigure}
    \hfill
    \begin{subfigure}[b]{0.32\textwidth}
        \centering
        \includegraphics[width=\textwidth]{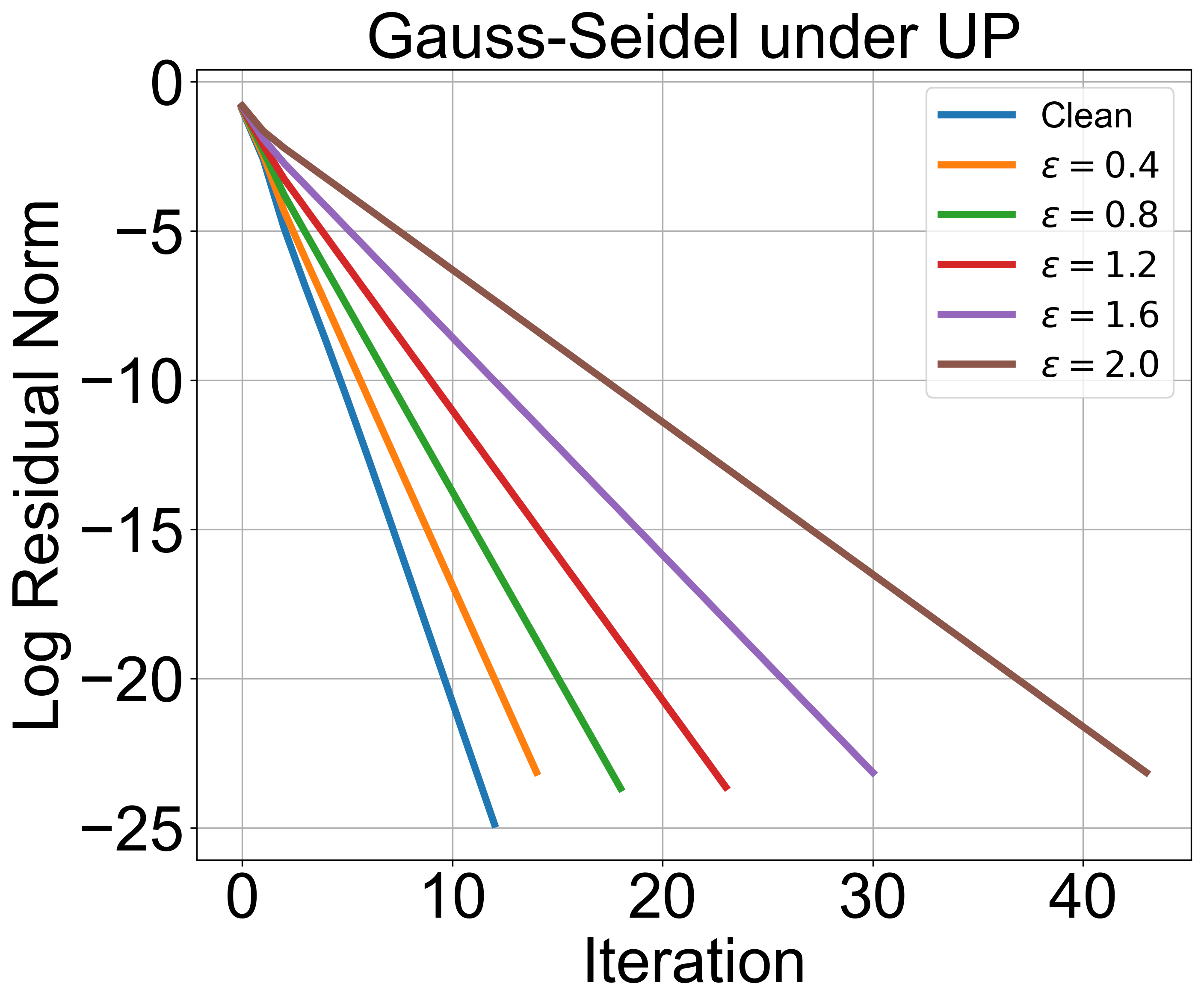}
        \caption{Gauss-Seidel (UP)}
        \label{fig:gauss_UP}
    \end{subfigure}
    \hfill
    \begin{subfigure}[b]{0.32\textwidth}
        \centering
        \includegraphics[width=\textwidth]{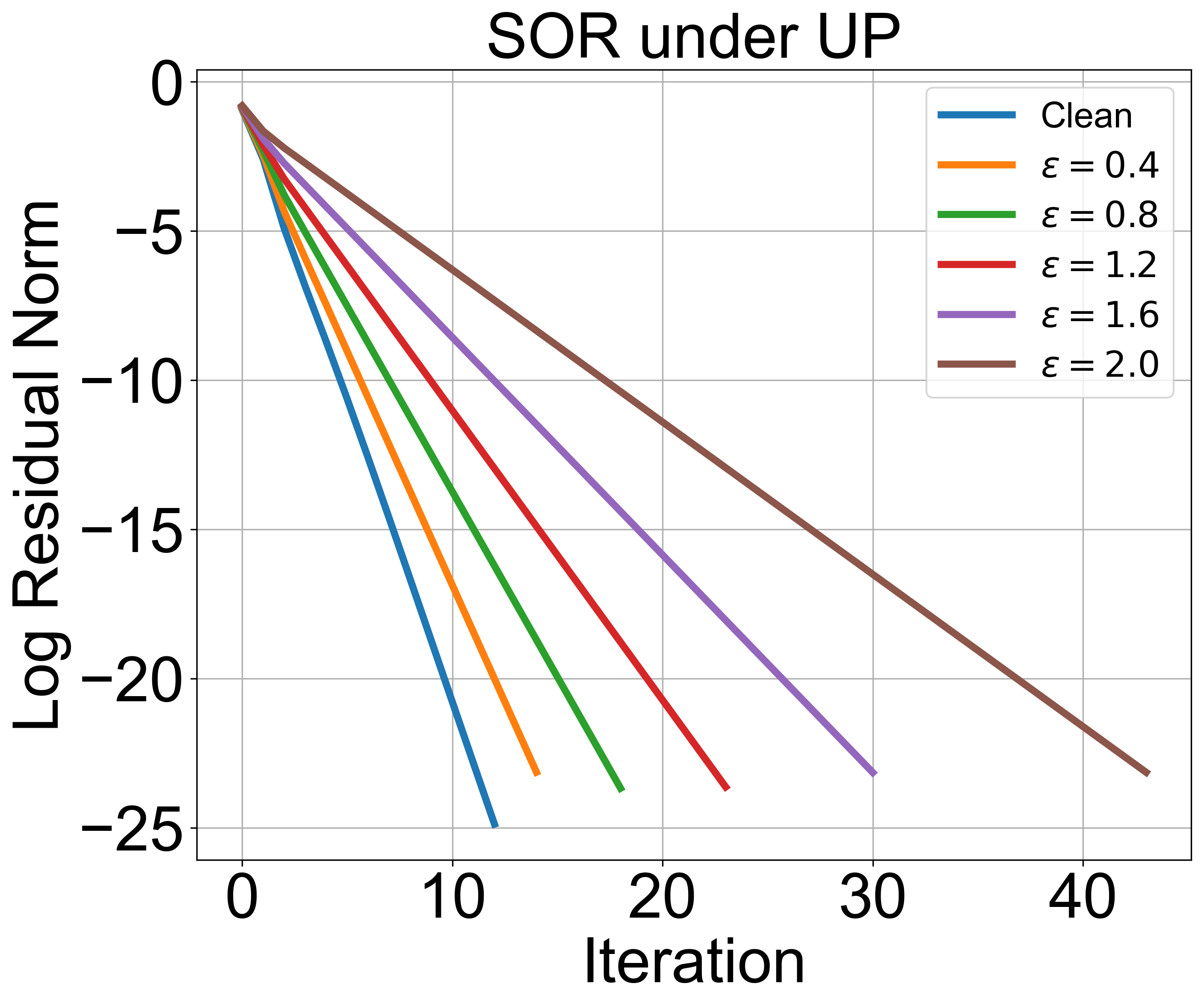}
        \caption{SOR (UP)}
        \label{fig:sor_UP}
    \end{subfigure}

    \begin{subfigure}[b]{0.32\textwidth}
        \centering
        \includegraphics[width=\textwidth]{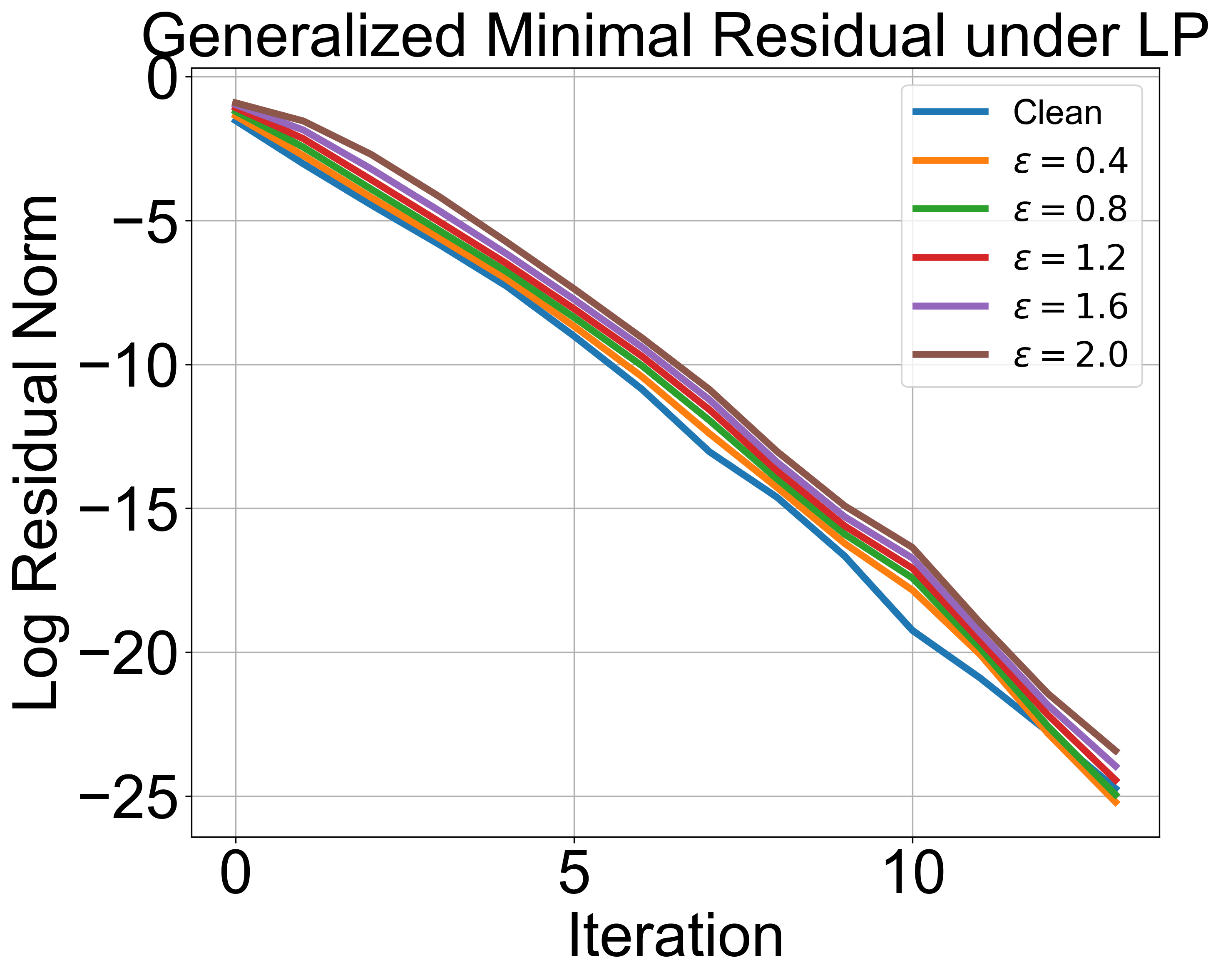}
        \caption{GMRES (LP)}
        \label{fig:gmres_LP}
    \end{subfigure}
    \hfill
    \begin{subfigure}[b]{0.32\textwidth}
        \centering
        \includegraphics[width=\textwidth]{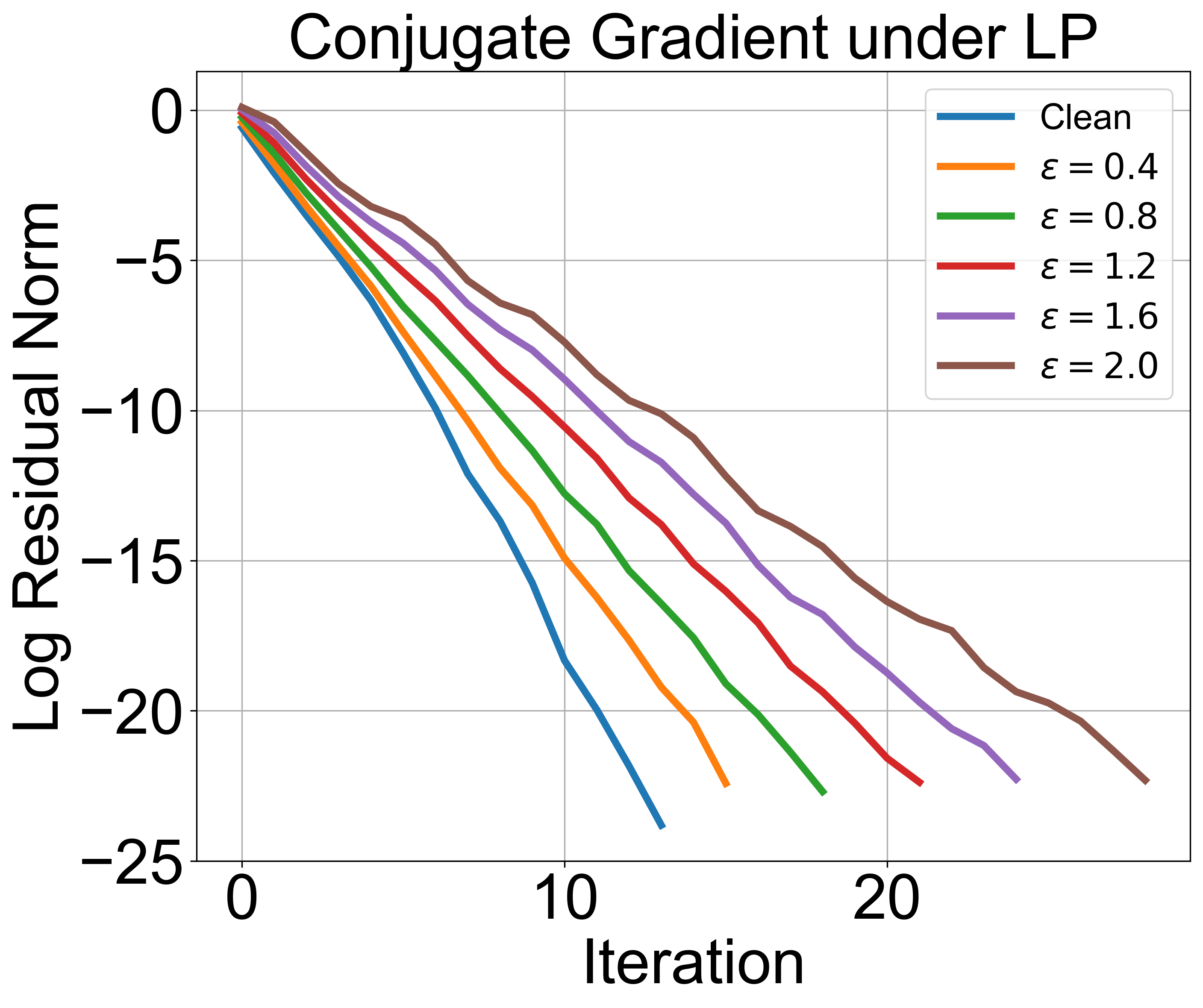}
        \caption{CG (LP)}
        \label{fig:cg_LP}
    \end{subfigure}
    \hfill
    \begin{subfigure}[b]{0.32\textwidth}
        \centering
        \includegraphics[width=\textwidth]{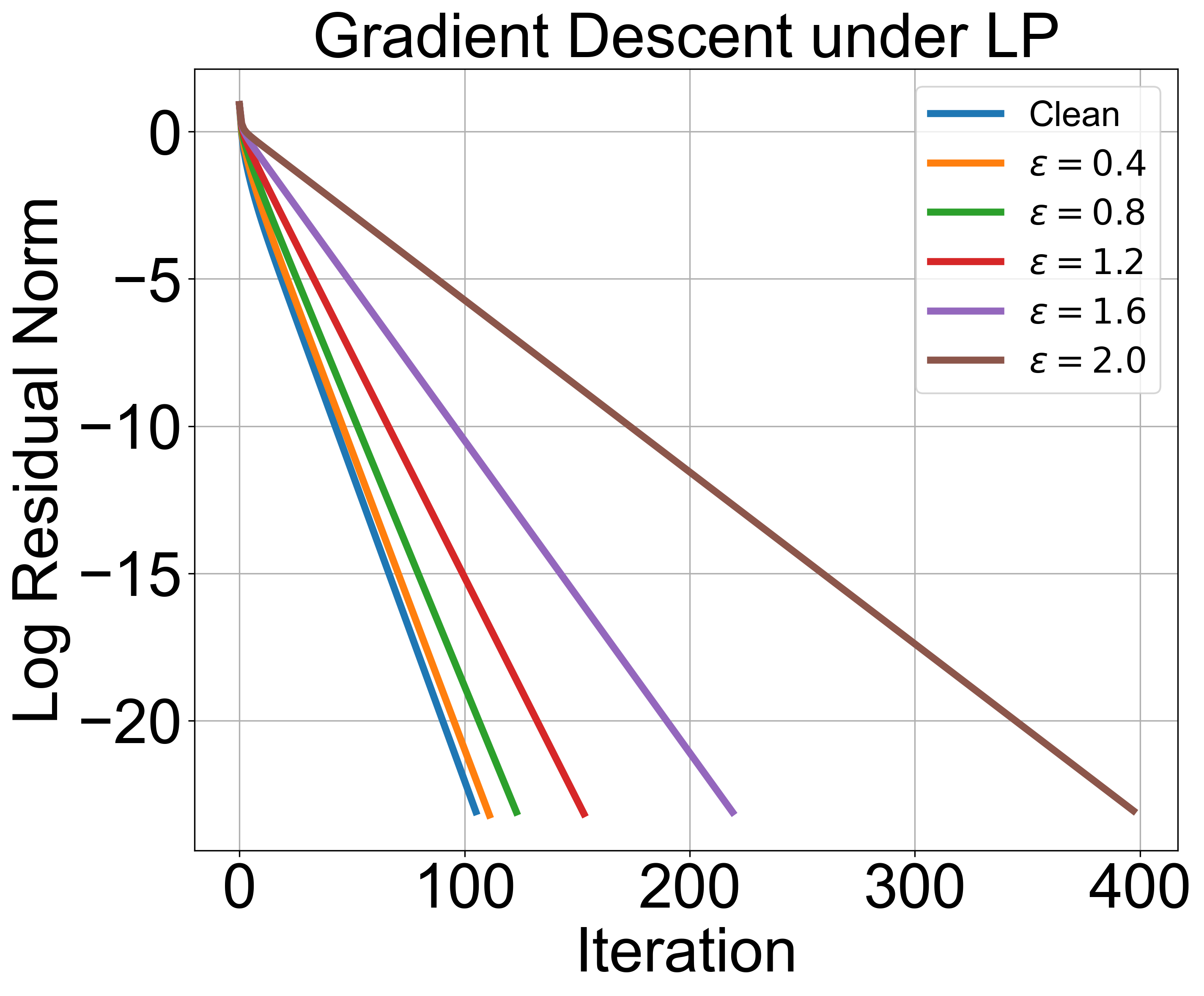}
        \caption{GD (LP)}
        \label{fig:gd_LP}
    \end{subfigure}

    \begin{subfigure}[b]{0.32\textwidth}
        \centering
        \includegraphics[width=\textwidth]{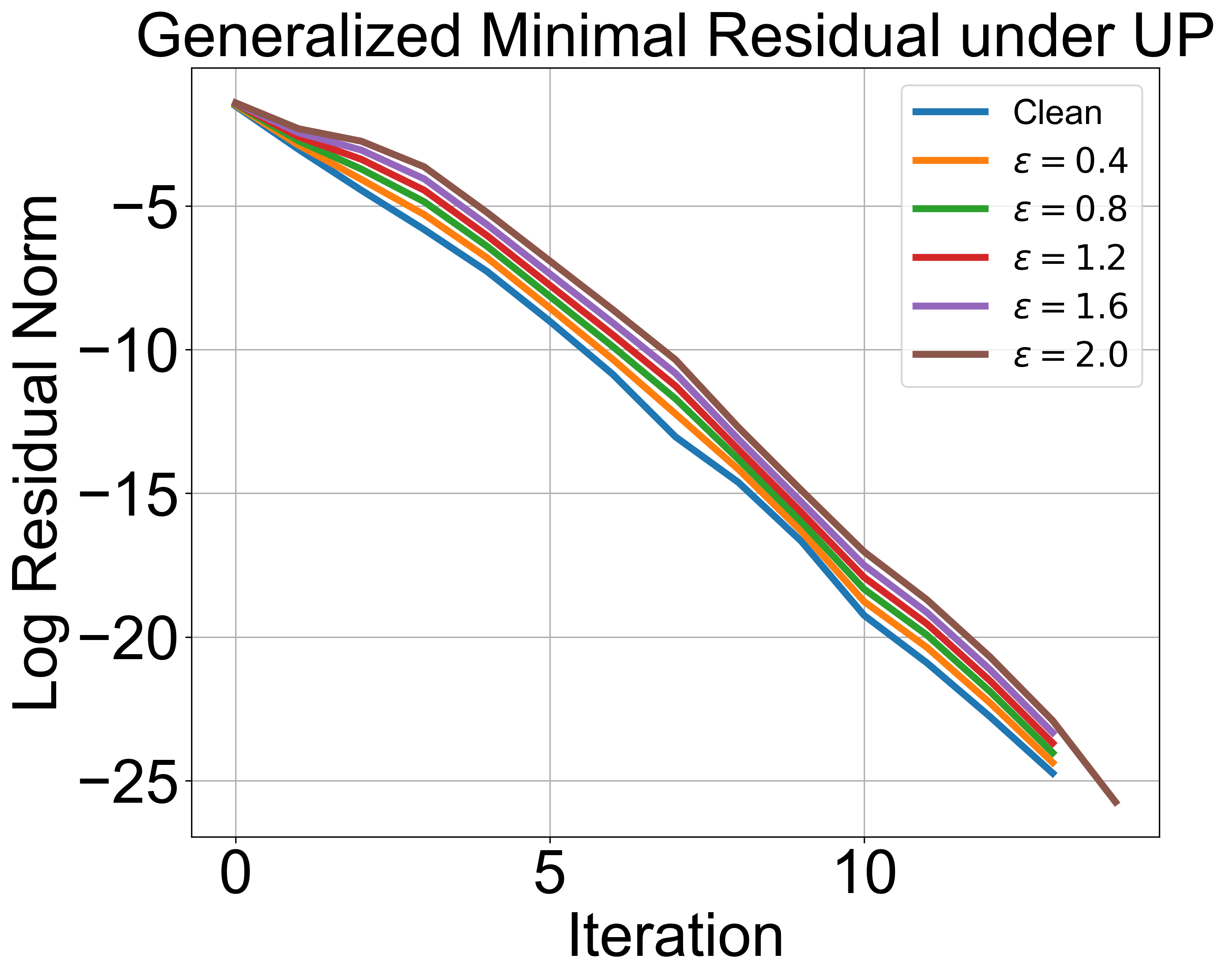}
        \caption{GMRES (UP)}
        \label{fig:gmres_UP}
    \end{subfigure}
    \hfill
    \begin{subfigure}[b]{0.32\textwidth}
        \centering
        \includegraphics[width=\textwidth]{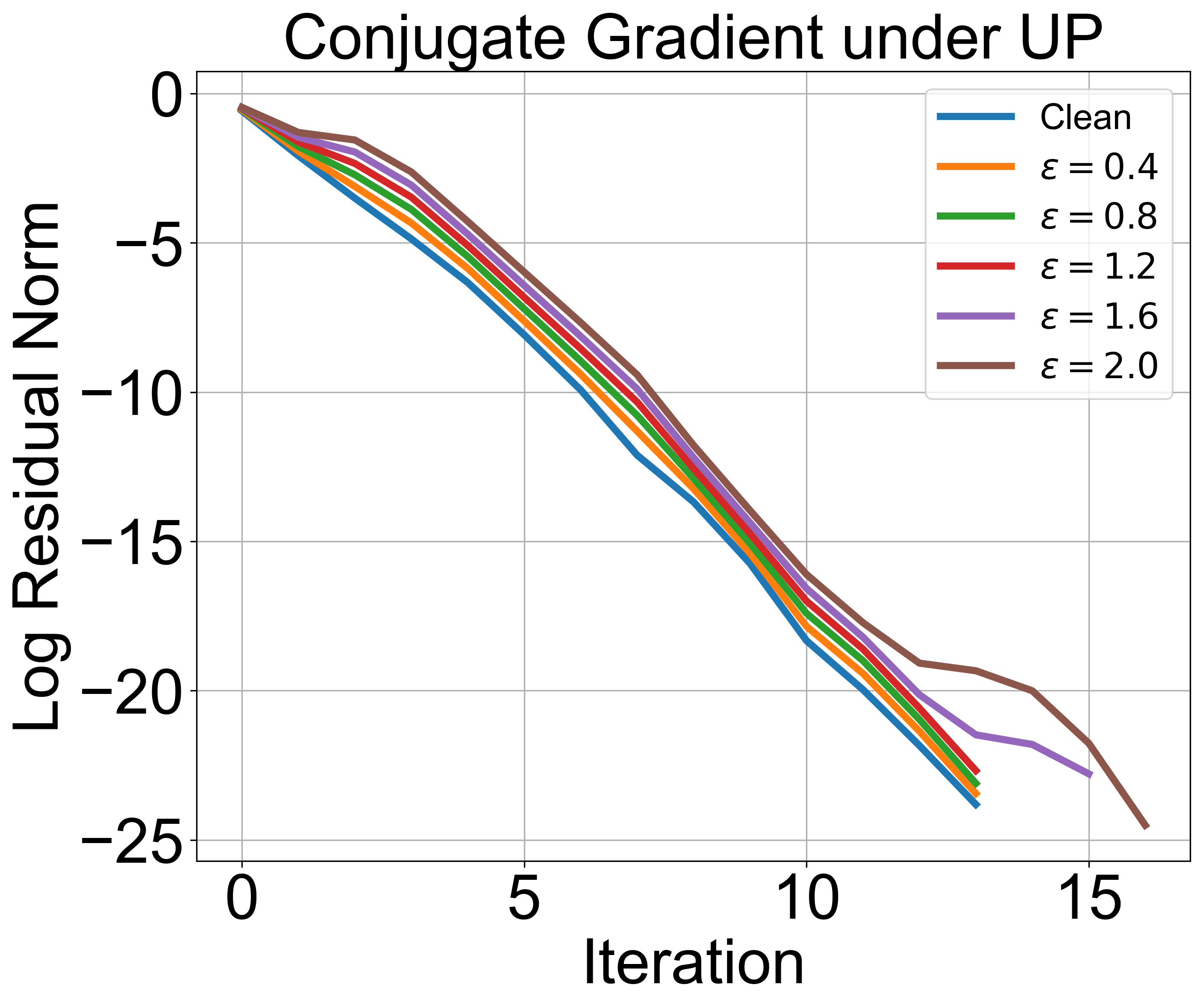}
        \caption{CG (UP)}
        \label{fig:cg_UP}
    \end{subfigure}
    \hfill
    \begin{subfigure}[b]{0.32\textwidth}
        \centering
        \includegraphics[width=\textwidth]{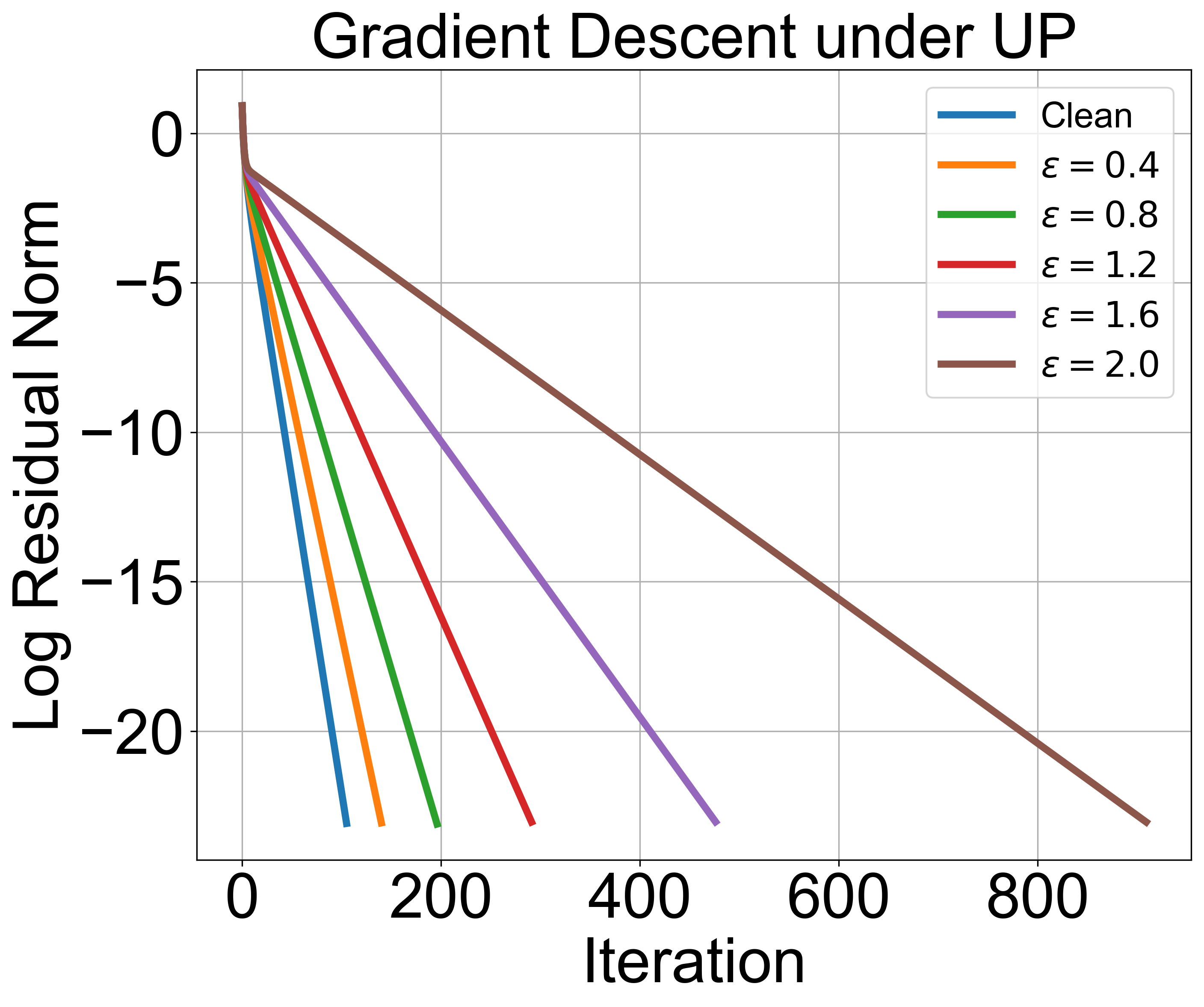}
        \caption{GD (UP)}
        \label{fig:gd_UP}
    \end{subfigure}

    \caption{Convergence behavior of six iterative solvers under two perturbations (LP and UP) with different $\epsilon$.}
    \label{fig:convergence_behavior}
\end{figure}

\section{Potential Broader Applications} 
Investigating poisoning attacks in linear systems empowers us to understand the underlying mechanism of poisoning attacks across various applications, including regression, classification, and distribution learning.

\textbf{Poisoning Attack to Linear Regression.} 
The first straightforward application is to investigate the poisoning attack to linear regressor $f=Xw$, where $w \in \mathbb{R}^d$ is the weight vector.
Specifically, we consider binary classification data composed of features $X \in \mathbb{R}^{n \times d}$ and labels $y \in  \mathbb{R}^n$. The loss function is the mean squared error (MSE) loss, i.e., the $l_2$ norm of the residual between the prediction and the ground truth, i.e., $L(w) = \frac{1}{n} \|Xw - y\|_2^2$. We also apply a regularization term to the loss function to avoid overfitting. Given a set of training data $(X_\text{train}, y_\text{train})$ and a set of testing data $(X_\text{test}, y_\text{test})$, we aim to find a small perturbation $\Delta X$ to the training data such that the linear regressor trained on the perturbed training data $(X_\text{train} + \Delta X, y_\text{train})$ cannot well predict on the testing data well. This application scenario is practical; for instance, some sensitive high-frequency trading data might be leaked to the public for unauthorized training. To prevent this, we can leverage poisoning attacks to add some noise to degrade the model generalization.

\textbf{Poisoning Attack to Linear Classifier.} 
The second application is to investigate the poisoning attack to the logistic regression model, i.e., $f(X) = \text{softmax}(Xw)$, where $w \in \mathbb{R}^d$ is the weight vector. The loss function is the cross-entropy loss, i.e., $L(w) = -\frac{1}{n} \sum_{i=1}^n y_i \text{logsoftmax}(X_i w)$. Our main goal is still to introduce some noise to the training data to degrade the generalization ability of the trained model. For example, in the domain of medical data, the medical data might be leaked to the public for unauthorized training. Despite the nonlinearity involved in the model part, the main framework of defining the loss function with respect to the perturbation and then leveraging the iterative optimization method to solve the problem is still applicable. The results provide valuable insights for solving this non-linear system.

\textbf{Poisoning Attack to Distribution Learning}. 
The third application is to investigate the poisoning attack to distribution learning. Specifically, we consider the distribution learning problem, where the goal is to learn a distribution $p_\theta$ that best fits the data distribution $p_\text{data}$. The distribution $p_\theta$ is parameterized by $\theta$. The loss function is defined as the KL divergence between the learned distribution $p_\theta$ and the data distribution $p_\text{data}$, i.e., $L(\theta) = \text{KL}(p_\text{data}||p_\theta)$. Taking text-to-image pair data $(X, Tx)$ as an example, the poisoning attack aims to inject some noise into the image data $X$ to degrade the generalization ability of the learned distribution $p_\theta$, which is trained on the perturbed data $(X + \Delta X, T)$. This application scenario is practical; for instance, to prevent subject-driven text-to-image synthesis on personal facial data, we can protect images by adding defensive noise \citep{le_etal2023antidreambooth,vsarvcevic2024u,deng2024survey}.

\section{Conclusions and Future Works}
% \section{Limitations }
% In this study, we first derive the $\ell_p$-norm forward and backward error bound for the linear system solvers, which helps us understand the worst-case perturbation that any perturbation policy might cause. Furthermore, from a data poisoning perspective, with two levels of knowledge, we formulate a poisoning problem as an optimization and propose two kinds of perturbation strategies, including the Label-guided Perturbation (LP) and the Unconditioning Perturbation (UP), to degrade the usability of perturbed data. 

This work analyzes the impact of data poisoning attacks on linear solvers in machine learning, focusing on the least squares problem. We adapt two standard perturbation strategies from poisoning techniques: Label-guided Perturbation (LP) and Unconditioning Perturbation (UP), and formulate them as optimization problems. Our research extends beyond traditional worst-case bounds, examining how various solvers respond to intentional perturbations and developing a theoretical framework for their convergence behavior under such attacks. Furthermore, we conduct theoretical analysis on the convergence rate and solution divergence lower bound on the gradient descent solver when using two proposed perturbation strategies. In the experiments, we conduct experiments on both direct and iterative solvers to verify the effectiveness of the proposed perturbation strategies. Results show that for the direct solver, the UP is more effective, while the LP is more suitable for attacking the iterative solver. Moreover, we empirically verify our forward error upper bound with hypothesis testing. In the analysis of convergence behavior, we observed that the perturbation leads to convergence slowdown for most iterative solvers for both LP and UP. We conduct further analysis of the root cause by linking them to the perturbation theory. Furthermore, we conduct preliminary exploration to leverage one of the pre-conditioning techniques, ILU factorization, for conducting data purification, which greatly stabilizes the convergence for both GMRES and CG methods. {In terms of limitations, currently, we only consider the linear system with a relatively small size. However, for larger problems, calculating the inverse and condition number of the matrix $X$ is computationally expensive. Therefore, how to leverage approximation to guide the perturbation strategy design is an interesting direction to explore in future work. 
% Another direction is to dive into the mechanism of the relationship of convergence manipulation using poisoning attacks, which is relatively under-explored in this work. 
}

% There are many interesting directions to explore based on the findings in this work, including:

% \header{Scaling to larger linear systems} Currently, we only consider the linear system with a relatively small size.
%  However, for larger problems, calculating the inverse and condition number of the matrix $X$ is computationally expensive. Therefore, we need to leverage approximation to guide the perturbation strategy design.

% \header{More Linear System Solvers} We only consider a few classical and basic linear system solvers in this work. There are many other more advanced solvers, such as the Generalized Minimal Residual Method (GMRES) \citep{saad1986gmres} and the Multigrid Method \citep{wesseling1995introduction}. How to design perturbation strategies for exploiting the weakness of these solvers and the corresponding enhancement is an interesting direction to explore. 

% \header{How perturbation affects the convergence} Typically, an increase in the condition number of a matrix can cause iterative methods to converge more slowly. However, in our works, we did not observe a consistent pattern of such a convergence behavior. We conjecture that that is because we mainly consider relatively small sparse matrices due to the limitation of calculating condition numbers. Future works might consider larger matrices and explore how these perturbation strategies affect the convergence behavior of different solvers in a more systematic way.

\bibliography{ref}
\bibliographystyle{abbrvnat}

% \clearpage
\appendix
% \section{Overview of Several Linear System Solvers}
% \label{sec:appendix_solver}
% \header{Normal Equation Solver}
% The Normal Equation Solver is used for solving linear regression problems. It is particularly effective when the number of features is not too large. The motivation for using this method is its simplicity and the fact that it does not require any feature scaling. It solves \(Ax = b\) with the formula:
% \[ x = (A^TA)^{-1}A^Tb, \]
% where \(A\) is the feature matrix and \(b\) is the output vector. This method is direct and does not require iterative steps, making it straightforward for small to medium-sized problems.

% \header{Jacobi Solver}
% The Jacobi Solver is an iterative method well-suited for large, sparse matrices. Its simplicity and the fact that it can be easily parallelized make it a good choice for high-dimensional problems. It updates the solution of \(Ax = b\) using:
% \[ x_i^{(k+1)} = \frac{1}{a_{ii}}\left(b_i - \sum_{j\neq i}a_{ij}x_j^{(k)}\right), \]
% making it particularly effective for diagonally dominant or well-conditioned matrices.

\section{Proof Appendix}

\subsection{Proof of Convergence Rate Lemma of L-smooth and Convex Function}

\begin{lemma}\label{lemma:smooth_convex} \citep{nesterov2013introductory}
If $f$ is $L$-smooth and convex, then $\forall x, y \in \mathbb{R}^d$, we have
\[
\langle \nabla f(x) - \nabla f(y), x - y \rangle \geq \frac{1}{L}\|\nabla f(x) - \nabla f(y)\|^2 \tag{5}
\]
\end{lemma}

\begin{lemma}\label{lemma:lemma2}
  Suppose that \( f : \mathbb{R}^d \to \mathbb{R} \) is \( L \)-smooth and convex and a finite solution \( x^\star \) exists. Consider the Gradient Descent Algorithm with \( 0 < \gamma \leq \frac{1}{L} \). We assume that \( \|x_0 - x^\star\|^2 \leq C \) for some \( C > 0 \), where \( x^\star \) is an optimal solution of \( f \). Then, we have
  \[
  \min_{0 \leq t \leq T-1} \left[ f(x_t) - f(x^\star) \right] \leq \frac{\|x_0 - x^\star\|^2}{\gamma T} \leq \frac{C}{\gamma} \cdot \frac{1}{T} \tag{8}
  \]
  or
  \[
  f(\tilde{x}_T) - f(x^\star) \leq \frac{\|x_0 - x^\star\|^2}{\gamma T} \leq \frac{C}{\gamma} \cdot \frac{1}{T} \tag{9}
  \]
  where \( \tilde{x}_T = \frac{1}{T} \sum_{t=0}^{T-1} x_t \).
\end{lemma}
\begin{proof}
  Suppose that the objective function \( f \) is \( L \)-smooth (\( L > 0 \)) and convex. Note that \( x_{t+1} := x_t - \gamma \nabla f(x_t) \). We have

\[
\|x_{t+1} - x^\star\|^2 = \|x_t - x^\star - \gamma \nabla f(x_t)\|^2
\]
\[
= \|x_t - x^\star\|^2 - 2\gamma \langle \nabla f(x_t), x_t - x^\star \rangle + \gamma^2 \|\nabla f(x_t)\|^2
\]
\[
= \|x_t - x^\star\|^2 - 2\gamma \langle \nabla f(x_t), x_t - x^\star \rangle + \gamma^2 \|\nabla f(x_t) - \nabla f(x^\star)\|^2
\]
\[
\leq \|x_t - x^\star\|^2 - 2\gamma \langle \nabla f(x_t), x_t - x^\star \rangle + \gamma^2 L \langle \nabla f(x_t) - \nabla f(x^\star), x_t - x^\star \rangle \tag{5}
\]
The last inequality is due to the lemma \ref{lemma:smooth_convex}. Rearranging the terms, we have

\[
\|x_{t+1} - x^\star\|^2 \leq \|x_t - x^\star\|^2 - 2\gamma \langle \nabla f(x_t), x_t - x^\star \rangle + \gamma^2 L \langle \nabla f(x_t) - \nabla f(x^\star), x_t - x^\star \rangle
\]
\[
= \|x_t - x^\star\|^2 - \gamma (2 - \gamma L) \langle \nabla f(x_t), x_t - x^\star \rangle - \gamma^2 L \langle \nabla f(x^\star), x_t - x^\star \rangle
\]
\[
= \|x_t - x^\star\|^2 - \gamma (2 - \gamma L) \langle \nabla f(x_t), x_t - x^\star \rangle.
\]

Hence, we have

\[
\|x_{t+1} - x^\star\|^2 \leq \|x_t - x^\star\|^2 - \gamma (2 - \gamma L) \langle \nabla f(x_t), x_t - x^\star \rangle. \tag{6}
\]

Since \( f \) is convex, we have

\[
f(y) \geq f(x) + \langle \nabla f(x), (y - x) \rangle, \quad x, y \in \mathbb{R}^d.
\]

Let \( x = x_t \) and \( y = x^\star \), we have

\[
f(x^\star) \geq f(x_t) + \langle \nabla f(x_t), (x^\star - x_t) \rangle
\]
\[
-\langle \nabla f(x_t), (x_t - x^\star) \rangle \leq -[f(x_t) - f(x^\star)]. \tag{7}
\]

Hence, from (6)

\[
\|x_{t+1} - x^\star\|^2 \leq \|x_t - x^\star\|^2 - \gamma (2 - \gamma L) \langle \nabla f(x_t), (x_t - x^\star) \rangle
\]
\[
\leq \|x_t - x^\star\|^2 - \gamma (2 - \gamma L) [f(x_t) - f(x^\star)].
\]

which is equivalent to

\[
[f(x_t) - f(x^\star)] \leq \frac{1}{\gamma (2 - \gamma L)} \left( \|x_t - x^\star\|^2 - \|x_{t+1} - x^\star\|^2 \right)
\]
Taking the sum from \( t = 0, \dots, T - 1 \), we have

\[
\sum_{t=0}^{T-1} [f(x_t) - f(x^\star)] \leq \frac{1}{\gamma (2 - \gamma L)} \left( \|x_0 - x^\star\|^2 - \|x_T - x^\star\|^2 \right)
\]
\[
\leq \frac{1}{\gamma (2 - \gamma L)} \|x_0 - x^\star\|^2
\]

Therefore,

\[
\frac{1}{T} \sum_{t=0}^{T-1} [f(x_t) - f(x^\star)] \leq \frac{1}{\gamma (2 - \gamma L)} \|x_0 - x^\star\|^2 \cdot \frac{1}{T}
\]

Taking the sum from \( t = 0, \dots, T - 1 \), we have

\[
\sum_{t=0}^{T-1} [f(x_t) - f(x^\star)] \leq \frac{1}{\gamma (2 - \gamma L)} \left( \|x_0 - x^\star\|^2 - \|x_T - x^\star\|^2 \right)
\]
\[
\leq \frac{1}{\gamma (2 - \gamma L)} \|x_0 - x^\star\|^2
\]

Therefore,

\[
\frac{1}{T} \sum_{t=0}^{T-1} [f(x_t) - f(x^\star)] \leq \frac{1}{\gamma (2 - \gamma L)} \|x_0 - x^\star\|^2 \cdot \frac{1}{T}
\]

With the assumption that \( \|x_0 - x^\star\|^2 \leq C \) for some \( C > 0 \), we have

\begin{equation}\label{eq:convergence_rate_l_smooth_convex}
  \frac{1}{T} \sum_{t=0}^{T-1} [f(x_t) - f(x^\star)] \leq \frac{C}{\gamma (2 - \gamma L)} \cdot \frac{1}{T}
\end{equation}

Note that if \( \gamma \leq \frac{1}{L} \),

\[
-\gamma(2 - \gamma L) = -2\gamma + \gamma^2 L \leq -2\gamma + \gamma = -\gamma,
\]
or
\[
\frac{1}{\gamma (2 - \gamma L)} \leq \frac{1}{\gamma}
\]

Then, we have

\[
\frac{1}{T} \sum_{t=0}^{T-1} [f(x_t) - f(x^\star)] \leq \frac{C}{\gamma} \cdot \frac{1}{T}
\]

We can have

\begin{equation}
  \label{eq:convergence_rate_l_smooth_convex_2}
  \min_{0 \leq t \leq T-1} \left[ f(x_t) - f(x^\star) \right] \leq \frac{1}{T} \sum_{t=0}^{T-1} \left[ f(x_t) - f(x^\star) \right] \leq \frac{C}{\gamma} \cdot \frac{1}{T}
\end{equation}

Let \( \tilde{x}_T = \frac{1}{T} \sum_{t=0}^{T-1} x_t \). Since \( f \) is convex, by Jensen's inequality, we have

\[
f(\tilde{x}_T) \leq \frac{1}{T} \sum_{t=0}^{T-1} f(x_t)
\]

Hence,

\[
f(\tilde{x}_T) - f(x^\star) \leq \frac{1}{T} \sum_{t=0}^{T-1} \left[ f(x_t) - f(x^\star) \right] \leq \frac{C}{\gamma} \cdot \frac{1}{T}
\]

\end{proof}

\subsection{Proof of Convergence Rate of $\alpha$-UP}
\label{app:theorem_alpha_up_proof}

\begin{theorem*}[\textbf{Convergence Rate of $\alpha$-UP for L-smooth and Convex Functions}]
 Let $f(w)$ be an objective function that is L-smooth and convex, e.g., $f(w; X) = \frac{1}{2} \| X w - y \|^2$, where $X$ is a data matrix. Denote the function after applying UP as $f{^\prime}(w; X+ \Delta X)$, which is also L-smooth and convex. Assume that an Unconditioning Perturbation (UP) is applied, increasing the condition number of the system such that the maximum singular value of the perturbed matrix is $\sigma_{\max}(X + \Delta X) = \alpha \sigma_{\max}(X)$, where $1 < \alpha < \sqrt{2/\gamma}$. Let $w^*$ be the optimal solution, and $w_T^\prime$ be the solution after $T^\prime$ iterations of gradient descent with step size $\gamma$. Then:
  
  \begin{enumerate}
      \item \textbf{Convergence Rate}: After applying UP, for an L-smooth and convex function, we have:
      \[
      \min_{0 \leq t \leq T-1} \left( f^\prime(w_t) - f^\prime(w^*) \right) \leq  \frac{C}{\gamma (2 - \alpha ^2 )} \frac{1}{T^\prime},
      \]
      where $\| w_0 - w^* \|^2 \leq C$ represents the initial distance between $w_0$ and the optimal solution $w^*$ is bounded by constant $C > 0$.
  
      \item \textbf{Total Iterations for $\beta$-accuracy}: The total number of iterations $T^\prime$ required to obtain a solution such that $f^\prime(w_T) - f^\prime(w^*) \leq \beta$ is:
      \[
      T^\prime \geq \frac{C}{\gamma (2 - \alpha ^2 ) } \cdot \frac{1}{\beta}. 
      \]
      % where $T$ is the total number of iterations required to achieve $\beta$-accuracy without UP.
  \end{enumerate}
\end{theorem*}

\begin{proof}
We begin by deriving how the change in the condition number propagates to a change in the smoothness constant $L$. For a matrix $X$, the condition number $\kappa(X)$ is defined as:

\[
\kappa(X) = \frac{\sigma_{\max}(X)}{\sigma_{\min}(X)},
\]

where $\sigma_{\max}(X)$ and $\sigma_{\min}(X)$ are the largest and smallest singular values of $X$, respectively. For a least-squares objective of the form $f(w) = \frac{1}{2} \| X w - y \|^2$, based on the definition of L-smoothness and gradient $\nabla f(w) = 2X^\top (X w - y)$, we have:

\[
\Vert \nabla f(w_1 ) - \nabla f(w_2) \Vert = \Vert 2X^\top (X w_1 - y) - 2X^\top (X w_2 - y) \Vert = \Vert 2X^\top X (w_1 - w_2) \Vert \leq L \Vert w_1 - w_2 \Vert
\]

Therefore, we know that $L$ is determined by the largest eigenvalue of the matrix $X^\top X$, which is equivalent to the square of the largest singular value of $X$:

\[
L = 2\sigma_{\max}(X)^2.
\]

When we apply Unconditioning Perturbation (UP) to the matrix $X$, it changes the maximum singular value such that $\sigma_{\max}(X + \Delta X) = \alpha \sigma_{\max}(X)$, where $\alpha > 1$. This directly affects the smoothness constant $L$ as follows:

\[
L_{\text{new}} = 2\sigma_{\max}(X + \Delta X)^2 = 2(\alpha \sigma_{\max}(X))^2 = 2\alpha^2 \sigma_{\max}(X)^2 = \alpha^2 L.
\]

Therefore, under Unconditioning Perturbation (UP), which increases the maximum singular value of the matrix $X$ by a factor $\alpha$, the new smoothness constant $L_{\text{new}}$ is:

\[
L_{f^\prime} = \alpha^2 L_f.
\]

% This shows that the smoothness constant $L$ grows quadratically with the factor by which the maximum singular value is increased. Any gradient-based method operating on the perturbed system will need to adjust its step size according to the new smoothness constant $L_{f^\prime} = \alpha^2 L_f$ to maintain convergence guarantees.

By applying Eq. \ref{eq:convergence_rate_l_smooth_convex} and Eq. \ref{eq:convergence_rate_l_smooth_convex_2} from Lemma \ref{lemma:lemma2}, we have 

\begin{equation}
  \label{eq:convergence_rate_alpha_up}
  \min_{0 \leq t \leq T-1} \left( f^\prime(w_t) - f^\prime(w^*) \right) \leq  \frac{C}{\gamma (2 - \gamma \alpha^2 )} \cdot \frac{1}{T^\prime} ,
\end{equation}

Under the assumption of bounded step size $\gamma \leq \frac{1}{L}$ and $1 < \alpha < \sqrt{2/\gamma }$, we have 

\begin{equation}
  \label{eq:convergence_rate_alpha_up_2}
  \min_{0 \leq t \leq T-1} \left( f^\prime(w_t) - f^\prime(w^*) \right) \leq  \frac{C}{\gamma (2 - \gamma \alpha^2 L)} \cdot \frac{1}{T^\prime} < \frac{C}{\gamma (2 - \gamma L )} \cdot \frac{1}{T^\prime}
\end{equation}

When imposing the $\beta$-accuracy, we have 

\begin{equation}
  \label{eq:convergence_rate_alpha_up_2}
  T^\prime \geq \frac{C}{\gamma (2 - \gamma \alpha^2 L)} \cdot \frac{1}{\beta}
  %  > \frac{C}{\gamma (2 - \gamma L)} \cdot \frac{1}{\beta} = T.
\end{equation}

\end{proof}

\subsection{Proof of Lower Bound on Solution Divergence of $\eta$-LP}
\label{app:theorem_eta_lp_proof}

\begin{theorem*}[Lower Bound on Divergence of Solution due to Label-guided Perturbations (LP)]
Let $w^*$ be the optimal solution for the system $X w = y$, and let $w'^*$ be the optimal solution for the perturbed system $(X + \Delta X) w'^* = y + \Delta y$. Assume that the perturbation $\Delta X$ is small, such that $\| \Delta X \|_2 \leq \epsilon$, and the perturbation induces a significant change in the prediction, with $\| \Delta y \|_2 \geq \eta$, where $\eta > 0$ is a threshold capturing the minimum size of the target change.

Then, the difference between the original solution $w^*$ and the perturbed solution $w'^*$ is lower-bounded by:
\[
\| w^* - w'^* \| \geq \frac{\eta}{\| X \|} \cdot \frac{1}{1 + \epsilon \| X^{-1} \|}.
\]
\end{theorem*}

\begin{proof}
  Consider the perturbed system:
  \[
  (X + \Delta X) w'^* = y + \Delta y.
  \]
  Subtracting the original system $X w^* = y$ from this equation, we get:
  \[
  X (w'^* - w^*) = \Delta y - \Delta X w'^*.
  \]
  Taking the norm of both sides yields:
  \[
  \| X (w'^* - w^*) \|_2 = \| \Delta y - \Delta X w'^* \|_2.
  \]
  Using the triangle inequality, this implies:
  \[
  \| X (w'^* - w^*) \|_2 \geq \| \Delta y \|_2 - \| \Delta X w'^* \|_2.
  \]
  
  We assume that $\| \Delta y \|_2 \geq \eta$, and that $\| \Delta X \|_2 \leq \epsilon$. Therefore, $\| \Delta X w'^* \|_2 \leq \epsilon \| w'^* \|_2$, giving:
  \[
  \| X (w'^* - w^*) \|_2 \geq \eta - \epsilon \| w'^* \|_2.
  \]
  
  Using the fact that $X$ is invertible, we can relate the norm $\| X (w'^* - w^*) \|_2$ to $\| w'^* - w^* \|_2$ by multiplying by $\| X^{-1} \|_2$:
  \[
  \| w'^* - w^* \|_2 \geq \frac{\eta - \epsilon \| w'^* \|_2}{\| X \|_2}.
  \]
  We assume that $\| w'^* \|_2 \leq \| X^{-1} \|_2 \| y \|_2$, so:
  \[
  \| w'^* - w^* \|_2 \geq \frac{\eta}{\| X \|_2} \cdot \frac{1}{1 + \epsilon \| X^{-1} \|_2}.
  \]
  
  This gives the lower bound on the divergence between the perturbed solution $w'^*$ and the original solution $w^*$, depending on the magnitude of the target change $\eta$ and the size of the perturbation $\epsilon$.
\end{proof}

\end{document}